\documentclass[10pt,twocolumn,letterpaper]{article}

\usepackage{cvpr}              

\usepackage{graphicx}
\usepackage{amsmath}
\usepackage{amssymb}
\usepackage{booktabs}
\usepackage{adjustbox}
\usepackage[accsupp]{axessibility}  

\usepackage[outercaption]{sidecap}

\usepackage[pagebackref,breaklinks,colorlinks]{hyperref}

\usepackage[capitalize]{cleveref}
\crefname{section}{Sec.}{Secs.}
\Crefname{section}{Section}{Sections}
\Crefname{table}{Table}{Tables}
\crefname{table}{Tab.}{Tabs.}

\usepackage{floatrow}

\def\cvprPaperID{9380} 
\def\confName{CVPR}
\def\confYear{2022}

\newcommand{\shortname}{\textsc{Lafite}}
\newcommand{\shortnamen}{$\textsc{Lafite}_{\textsc{NN}}$}
\newcommand{\shortnameg}{$\textsc{Lafite}_{\text{G}}$}
\newcommand{\longname}{\textbf{LA}nguage-\textbf{F}ree~tra\textbf{I}ning~for~\textbf{T}ext-to-image~g\textbf{E}neration}

\newcommand{\lafetilogo}{\includegraphics[width=10px]{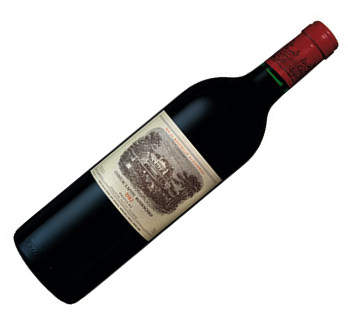}}

\usepackage{amsmath,amsfonts,bm}

\def\eqref#1{equation~\ref{#1}}

\def\1{\bm{1}}

\def\rvc{{\mathbf{c}}}

\def\rvh{{\mathbf{h}}}
\def\rvu{{\mathbf{i}}}

\def\rvs{{\mathbf{s}}}
\def\rvt{{\mathbf{t}}}
\def\rvu{{\mathbf{u}}}

\def\rvw{{\mathbf{w}}}
\def\rvx{{\mathbf{x}}}

\def\rvz{{\mathbf{z}}}

\DeclareMathAlphabet{\mathsfit}{\encodingdefault}{\sfdefault}{m}{sl}
\SetMathAlphabet{\mathsfit}{bold}{\encodingdefault}{\sfdefault}{bx}{n}

\def\gS{{\mathcal{S}}}

\def\gU{{\mathcal{U}}}

\def\gW{{\mathcal{W}}}

\def\gZ{{\mathcal{Z}}}

\usepackage{adjustbox}
\usepackage{booktabs}
\usepackage{hyperref}
\usepackage{url}
\usepackage{algorithm}
\usepackage{algorithmic}
\usepackage{wrapfig}
\usepackage{multirow}
\usepackage{xcolor}
\usepackage{amsthm}
\ifx\proof\undefined
\newenvironment{proof}{\par\noindent{\bf Proof\ }}{\hfill\BlackBox\\[2mm]}
\fi

\ifx\theorem\undefined
\newtheorem{theorem}{Theorem}
\fi
\ifx\example\undefined
\newtheorem{example}{Example}
\fi
\ifx\lemma\undefined
\newtheorem{lemma}[theorem]{Lemma}
\fi

\ifx\corollary\undefined
\newtheorem{corollary}[theorem]{Corollary}
\fi

\ifx\assumption\undefined
\newtheorem{assumption}{Assumption}
\fi

\ifx\definition\undefined
\newtheorem{definition}{Definition}
\fi

\ifx\proposition\undefined
\newtheorem{proposition}[theorem]{Proposition}
\fi

\ifx\remark\undefined
\newtheorem{remark}{Remark}
\fi

\ifx\conjecture\undefined
\newtheorem{conjecture}[theorem]{Conjecture}
\fi

\ifx\factoid\undefined
\newtheorem{factoid}[theorem]{Fact}
\fi

\ifx\axiom\undefined
\newtheorem{axiom}[theorem]{Axiom}
\fi

\renewcommand{\eqref}[1]{\textup{{\normalfont(\ref{#1}}\normalfont)}}

\title{\shortname{} \includegraphics[width=18px]{figures/Lafite.png} : Towards Language-Free Training for Text-to-Image Generation}

\author{Yufan Zhou $^1$,
 ~
Ruiyi Zhang $^2$,
 ~
Changyou Chen $^1$,
 ~
Chunyuan Li $^3$,\\
 ~
Chris Tensmeyer $^2$,
 ~
Tong Yu $^2$,
 ~
Jiuxiang Gu $^2$,
 ~
Jinhui Xu $^1$\thanks{The research of the first and eighth authors was supported in part by NSF through grants IIS-1910492.},
 ~
Tong Sun $^2$\\

$^1$ State University of New York at Buffalo~~~~ 
$^2$ Adobe Research~~~~
$^3$ Microsoft Research, Redmond\\

{\tt\small \{yufanzho, changyou, jinhui\}@buffalo.edu}\\
{\tt\small \{ruizhang, tensmeye, tyu, jigu, tsun\}@adobe.com}~~~~ 
{\tt\small chunyl@microsoft.com}
}

\begin{document}
\maketitle

\begin{abstract}
    One of the major challenges in training text-to-image generation models is the need of a large number of high-quality image-text pairs. While image samples are often easily accessible, the associated text descriptions typically require careful human captioning, which is particularly time- and cost-consuming. In this paper, we propose the first work to train text-to-image generation models {\em without any text data}. Our method leverages the well-aligned multi-modal  semantic space of the powerful pre-trained CLIP model: the requirement of text-conditioning is seamlessly alleviated via generating text features from image features. Extensive experiments are conducted to illustrate the effectiveness of the proposed method. We obtain state-of-the-art results in the standard text-to-image generation tasks. Importantly, the proposed language-free model outperforms most existing models trained with full image-text pairs. Furthermore, our method can be applied in fine-tuning pre-trained models, which saves both training time and cost in training text-to-image generation models.
   Our pre-trained model obtains competitive results in zero-shot text-to-image generation on the MS-COCO dataset, yet with around only 1\% of the model size and training data size relative to the recently proposed large DALL-E model.
\end{abstract}

\section{Introduction}

Automatic synthesis of realistic images from arbitrary text description is one of the core aspirations in artificial intelligence.
Most existing works achieve the goal by consuming a large number of high quality image-text pairs~\cite{xu2018attngan, zhu2019dm, zhang2021crossmodal, ramesh2021zero, ding2021cogview}, which, however, often requires heavy workload of precise human captioning and filtering. For instance, MS-COCO~\cite{lin2014microsoft}, the most commonly used dataset in text-to-image generation tasks, requires over 70,000 worker hours in gathering and annotating the captions. Even for less curated datasets such as Google Conceptual Captions \cite{Sharma2018ConceptualCA}, it consists of 3.3 million image-text pairs that are heavily filtered from 5 billion images from around 1 billion English webpages. In practice, for a customized domain, it is infeasible to collect such a large number of image-text pairs for model training, due to the high cost of human captioning and filtering. This challenge renders the unprecedented importance of the zero-shot text-to-image generation tasks, where no domain-specific image-text pairs are used to train a model to generate images in a given domain. 



Recently, several attempts have been made to tackle zero-shot text-to-image generation problem, by pre-training giant generative models on web-scale image-text pairs, such as  DALL-E \cite{ramesh2021zero} and CogView \cite{ding2021cogview}. Both are auto-regressive Transformer models built for zero-shot text-to-image generation, as they can generate corresponding images given arbitrary text description without training on domain-specific datasets. However, to ensure good performance, these models require a gigantic scale of  {\em data collections}, {\em model size} and {\em model training}. Specifically, DALL-E contains over 12 billion parameters and is trained on a dataset consisting of 250 million image-text pairs; CogView is a model with 4 billion parameters trained on 30 million image-text pairs. 
For this reason, hundreds of GPUs are required in training these models, which significantly increases carbon footprint and decrease the inclusivity: making it extremely difficult for more researchers to participate the study of this topic.


It is therefore desired to provide affordable solutions to build text-to-image generation models for the settings of limited image-text pair data, by reducing the requirements on model size, data collections and model training. In terms of data collections, in the ideal scenarios, the \textit{language-free} setting is probably the minimal and cheapest requirement, where only image data is provided.
This is important because collecting only image data is much easier than constructing high-quality image-text pairs, given the ample domain-specific image datasets available online.

\begin{figure*}[t!]
	\vspace{-0mm}\centering
	\begin{tabular}{c c c}
		\hspace{-4mm}
		\includegraphics[height=3.05cm]{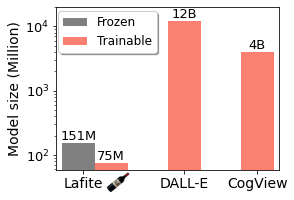}  & 
		\hspace{-0mm}
		\includegraphics[height=3.05cm]{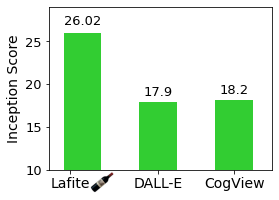}  & 
		\includegraphics[height=3.05cm]{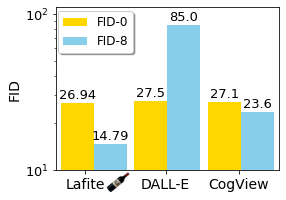}  \\
	\end{tabular}
	\vspace{-3mm}
	\caption{Model size {\em vs} performance of zero-shot image-to-text generation on the COCO dataset. \shortname{} has much smaller model size, especially when considering trainable parameters (Left figure), but shows higher Inception score (Middle figure) and lower FID  (Right figure). Please refer to Section~\ref{sec:exp} for details.
	 }
	\vspace{-0mm}
	\label{fig:hero_cmp}
\end{figure*}

To this end, we propose \shortname{}\footnote{\lafetilogo{} \longname{}}\lafetilogo{}, a generative adversarial approach to significantly lowering the cost barrier and to building efficient text-to-image generation models, based on the pre-trained CLIP model \cite{radford2021learning}. 
Specifically, $(i)$ we take advantages of CLIP's property on image-text feature alignment in the joint semantic space, to construct  pseudo image-text feature pairs; 
 $(ii)$ we propose a text-to-image GAN (Generative Adversarial Network) model~\cite{goodfellow2014generative} that can effectively leverage pseudo image-text feature pairs. Our major contributions can be summarized as followings:


\begin{itemize}
    \item We propose \shortname{}, a versatile system that works effectively in a large range of text-to-image generation settings, including language-free, zero-shot and fully-supervised learning. 
    
    \item To the best of our knowledge, \shortname{} is the first work that enables the language-free training for the text-to-image generation task. We propose two novel schemes to construct pseudo image-text feature pairs, and conduct comprehensive study for the new setting. The effectiveness is validated with quantitative results on several datasets with different training schemes (training from scratch and fine-tuning from pre-trained generative models).
    
    \item In zero-shot text-to-image generation settings, \shortname{} outperforms the prior art  DALL-E and CogView on the COCO benchmark, with less than 1\% of the trainable model parameter size (with frozen CLIP model weights). Please see Figure~\ref{fig:hero_cmp} for comparisons.
    \item In the standard fully supervised settings, \shortname{} outperforms several state-of-the-art (SoTA) methods by a large margin. Surprisingly, even our language-free model shows superior performance than most existing models that are trained with full image-text pairs.
    
    
\end{itemize}

\section{Related Work}

\paragraph{Text-to-image generation}
Existing models on text-to-image generation can be categorized into two classes:
fully-supervised text-to-image generation \cite{xu2018attngan,zhu2019dm, zhang2021crossmodal} 
and
zero-shot text-to-image generation \cite{ramesh2021zero, ding2021cogview}. 
The SoTA in the full image-text pair setting is still dominated by GAN variants~\cite{xu2018attngan,zhu2019dm,zhang2021crossmodal}. GANs~\cite{goodfellow2014generative} have inspired many advances in image synthesis~\cite{mirza2014conditional,karras2017progressive,liu2017unsupervised,li2017alice,karras2019analyzing}. For text-to-image synthesis, the improved model performance is often benefited from large generative adversarial image models~\cite{zhang2021crossmodal} and pre-trained text encoders~\cite{liu2019roberta}. Recently, excellent zero-shot text-to-image generation performance has been achieved in DALL-E \cite{ramesh2021zero} and CogView \cite{ding2021cogview}. The basic idea is to encode images into discrete latent tokens using VQ-VAE\cite{van2017neural, razavi2019generating}, and pre-train a huge-size auto-regressive Transformers\cite{vaswani2017attention} to predict these discrete tokens based on paired text sequences. Our \shortname{} is the first generative adversarial approach that achieves SoTA on zero-shot generation.

\paragraph{Multi-modal feature learning}
Learning a joint and aligned feature space for vision-and-language has been a long standing problem in artificial  intelligence~\cite{weston2010large,socher2010connecting}. Inspired by the BERT model~\cite{devlin2018bert}, a number of methods attempt to learn generic multi-modal fusion layers, given the pre-extracted visual region features  and textual encoder~\cite{lu2019vilbert,li2020oscar,su2019vl,zhang2021vinvl,kim2021vilt,li2021align}. These works aim at learning generic multi-modal representations for downstream tasks like visual question answering~\cite{antol2015vqa,hudson2019gqa}, image captioning~\cite{lin2014microsoft,agrawal2019nocaps}, visual commonsense reasoning~\cite{zellers2019recognition}. Unlike the aforementioned works, another line of works focus on the way of learning visual representation from natural language supervisions, including both generative~\cite{desai2021virtex} and discriminative~\cite{wang2016learning, wang2018learning, zhang2020contrastive} methods. The latter learns an aligned visual-semantic space. This idea is recently scaled up in CLIP/ALIGN~\cite{radford2021learning, jia2021scaling}, which pave the way toward building a {\em universal} image-text representation space. Our \shortname{} is built up in this universal space, and is the first one to leverage its multi-modal alignment property for language-free text-to-image generation.


\paragraph{CLIP for generation/manipulation.} The idea of multi-modal feature space also inspires some recent works on generative models~\cite{galatolo2021generating,patashnik2021styleclip,gal2021stylegan,pakhomov2021segmentation}. All of these works are related to ours in that the tools of pre-trained CLIP model and StyleGAN2 are employed. Our \shortname{} is different in two aspects:
$(i)$ The motivations and scenarios are different. Existing works focus on latent optimization~\cite{galatolo2021generating}, image manipulation~\cite{patashnik2021styleclip}, domain adaptation~\cite{gal2021stylegan}, image segmentation~\cite{pakhomov2021segmentation}. We present the first study on training text-to-image generation models without the requirement of paired captions.
$(ii)$ The techniques are different. Though image-text feature alignment property is leveraged in all works, Our \shortname{} is the only one to generate pseudo features pairs in the joint multi-modal space, none of existing works considers such a possibility.

\section{\shortname{}: A Language-Free Paradigm}

A natural idea to avoid human captioning in constructing image-text pair training data is using an off-the-shelf image captioning model that can automatically generate captions for the collected training images. However, this is especially challenging due to the lack of a universal captioning model that can $(i)$ bridge the modality gap between text and image to generate high-quality captions; $(ii)$ generalize to diverse image domains with large domain gaps. In this paper, we resort to solving an easier problem: one may directly {\em generate text features} rather than text descriptions, to avoid the use of image captioning models.

Throughout the paper, $(\rvx, \rvt)$ denotes an image-text pair, $\rvx^\prime$ is the corresponding generated image of $\rvt$. $G$ and $D$ denote the generator and discriminator respectively.
We use $f_{\text{img}}$ and $f_{\text{txt}}$ to denote the pre-trained text encoder and image encoder, which map text descriptions and image samples into a joint multi-modal feature space.  $\rvh=f_{\text{txt}}(\rvt)$ denotes the real text feature, $\rvz \sim \mathcal{N}(\mathbf{0}, \mathcal{I})$ denotes latent noise sampled from the standard Gaussian distribution, serving as one input of the generator. Our idea to achieve language-free training is to generate pseudo text features $\rvh^\prime$, which aims to approximating $\rvh$, by leveraging the image-text feature alignment of a pre-trained model. The generated features are then fed into the text-to-image generator to synthesize the corresponding images. Without loss of generality, we denote the mapping from input data to the  multi-modal feature space as {\em translator} $T$ in two settings. If only images $\rvx$ are provided (\ie language-free setting), we consider a  pseudo text-feature generation process $T:\rvx\rightarrow \rvh^\prime$; If image-text pairs  $(\rvx, \rvt)$ are provided (\ie standard fully-supervised settings), we  encode ground-truth text, $T: \rvt \rightarrow \rvh$.

\subsection{Pseudo Text-Feature Generation}\label{sec:text_feature}


To achieve the goal, a universal multimodal feature space is desired, where features of paired texts and images are well aligned. The recently  vision-and-language models such as CLIP and ALIGN achieve this, by pre-training on hundreds/thousands of millions of image-text pairs using contrastive learning. The cosine similarity between matched image-text features is maximized, while cosine similarity of the mis-matched pair is minimized. This naturally provides a high-dimensional hyper-sphere\footnote{In our implementation, we normalize the features extracted with CLIP by their L2 norm.}
for the multimodal features,  where paired image-text should be close to each other, with a small angle between their feature vectors.  
This inspires us to explore the potentials of generating pseudo text features $ \rvh^\prime \in \mathcal{H}(\rvx)$ for a given image $\rvx$ on this hyper-sphere: $\mathcal{H}(\rvx) = \{ \rvh^\prime | \mathrm{Sim}(\rvh^\prime, f_{\text{img}}(\rvx)) \geq c \}$, where $\mathrm{Sim}$ denotes cosine similarity, $c>0$ is a threshold. This idea is illustrated in Figure~\ref{fig:clip_feature}. Based on the analysis, we consider two schemes to generate pseudo text features.


\begin{SCfigure}
\includegraphics[width=0.36\linewidth]{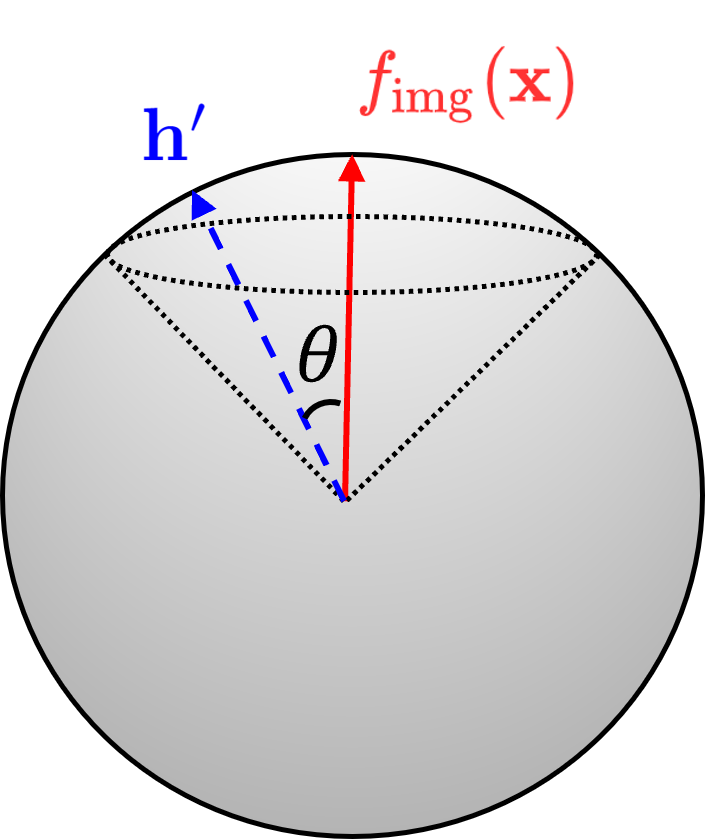} \hspace{0.2cm}
\caption{The illustration that the generated pseudo text feature vector $\rvh^\prime \in \mathcal{H}(\rvx)$ ({\color{blue}blue dashed arrow}) should have high cosine similarity with the image feature $f_{\text{img}}(\rvx)$ ({\color{red}red solid arrow}), \ie $\theta \leq \arccos{c}$.
    } \label{fig:clip_feature}
\end{SCfigure}




\paragraph{Fixed perturbations}


To generate pseudo text feature $\rvh^\prime$,
we propose to perturb the image feature $ f_{\text{img}}(\rvx)$ with adaptive Gaussian noise:
\begin{align}\label{eq:adaptive_gaussian}
    \rvh^\prime = \tilde{\rvh} / \Vert \tilde{\rvh} \Vert_2, ~~~ \tilde{\rvh} = f_{\text{img}}(\rvx) + \xi \mathbf{\epsilon} \Vert  f_{\text{img}}(\rvx) \Vert _2  /\Vert \mathbf{\epsilon}\Vert_2,
\end{align}
where $\mathbf{\epsilon} \sim \mathcal{N}(\mathbf{0}, \mathbf{I})$ is the Gaussian noise, $\xi>0$ is a fixed hyper-parameter representing the level of perturbations, $\Vert \cdot \Vert_2$ denotes L2 norm.
The added Gaussian noise is {\em adaptive} in the sense that it is normalized to a hyper-sphere, then re-scaled by the norm of image feature. 
We can prove that, with the adaptive noise, our \shortnameg{} can generate $\mathcal{H}(\rvx)$ with a high probability which depends on $\xi, c$ and $d$.  The formal theorem and its proof are provided in the Appendix.

\paragraph{Trainable perturbations}
It is natural to extend \shortnameg{} to learn more adaptive noise instead of using a vanilla Gaussian.
To this end, we propose to train an {\em inference} model which takes the image features as inputs and outputs the mean and variance of the desired noise distribution. Specifically, the inference model consists of two neural networks $r_{1}(\cdot)$ and $ r_{2}(\cdot)$. 
With the re-parameterization trick \cite{kingma2013auto}, the generation of pseudo text features is:
\begin{align}\label{eq:txt_fts_nn}
    \rvh^\prime & = \tilde{\rvh} / \Vert \tilde{\rvh} \Vert_2, \text{where}  \\ \tilde{\rvh} & = f_{\text{img}}(\rvx) + r_{1}(f_{\text{img}}(\rvx)) + \mathbf{\epsilon} \odot \exp(r_2(f_{\text{img}}(\rvx))),~ \nonumber
\end{align}
where $\exp$ denotes element-wise exponent operation, and $\odot$ denotes element-wise multiplication, $\mathbf{\epsilon} \sim \mathcal{N}(\mathbf{0}, \mathbf{I})$ denotes noise sampled from standard Gaussian. In practice, we construct $r_{1}(\cdot) $ and $ r_{2}(\cdot)$ with 4 fully-connected (FC) layers respectively, and train them in a supervised way by maximizing the cosine similarity $\mathrm{Sim}(\rvh^\prime, \rvh)$ between generated text features and real text features.

\paragraph{Discussion.} Both schemes have their own pros and cons. The trainable perturbation generally yields better performance than the fixed perturbation. However, the fixed perturbation is easier to use, without the requirement of training an inference model on an additional dataset with annotated image-text pairs. Further, the performance of trainable perturbation is influenced by the gap between datasets used in training the inference model and the generative model, as empirically verified in our ablation studies in the experiment section.


\subsection{Network Architectures}

We propose to adapt the unconditional StyleGAN2 to a conditional generative model for our goal. 
Note that although we discuss our model in a language-free setting, it can be directly generalized to standard text-to-image generation by using $\rvh$ (real text feature) instead of $\rvh^\prime$ (pseudo text feature). 
\paragraph{Generator} It is shown in recent works \cite{liu2020style, wu2021stylespace} that the \textit{StyleSpace} of StyleGAN2 is a well-disentangled intermediate feature space, whose dimensions are highly independent. By leveraging this property, we propose a simple yet effective approach to enable conditional generation: injecting new conditional information directly into the StyleSpace, as illustrated in Figure~\ref{fig:generator}. Specifically, we choose to inject text information as follows. 
$(i)$ 
Random noise vectors $\rvz \in \gZ $ are transformed into an intermediate latent space $\gW$ via a so-called mapping network, which consists of a sequence of FC layers.
The $\gW$ space is claimed to better reflect the disentangled nature of the learned distribution. Each 
$\rvw \in \gW $ is further transformed to channel-wise {\em unconditional style codes} $\rvs$, using a different learned affine transformation for each layer of the generator. The space spanned by these style parameters is often  referred to as StyleSpace, or $\gS$.
$(ii)$ For a conditional vector $\rvh^\prime$ from the image-text joint semantic space of CLIP, it is transformed into {\em condition codes} $\rvc$, using a different learned 2-layer FC network for each generator layer. 
$(iii)$ At each layer of the generator, we concatenate its style and conditional codes to obtain $[\rvs, \rvc]$, which is 
is further transformed to channel-wise {\em conditional style codes} $\rvu$ , using a different learned affine transformation for each generator layer. We refer to the space spanned by these style parameters as {\em  Conditional StyleSpace}, or $\gU$. In sum, the generator $G$ synthesizes a fake image as:
\begin{align}\label{eq:generator}
    \rvx^\prime = G( \rvh^\prime, \rvz) 
\end{align}
\paragraph{Discriminator} In the text-to-image task, the discriminator ensures the generated image to satisfy two criterias: photo-realistic to human perception and fidelity to the text condition. 
To this end, we encode the input image $\rvx$ with a shared discriminator backbone, then perform two tasks (each with a task-specific FC layer), as illustrated in Figure \ref{fig:training_obj}. 
$(i)$ $f_{\text{d}}(\rvx)$ projects $\rvx$ into a scalar, indicating the level of true or fake of an input image $\rvx$. This is a common task shared in all GAN models;
$(ii)$ $f_{\text{s}}(\rvx)$ 
embeds $\rvx$ into a semantic space, which is expected to be similar to the semantic space of CLIP. We compute the inner product $\langle \rvh^\prime, f_{\text{s}}(\rvx)\rangle$ to indicate how well the input image $\rvx$ is semantically aligned/conditioned with the pseudo text feature. In summary, the discriminator output is defined as:
\begin{align}\label{eq:logit}
    D(\rvx, \rvh^\prime) = 
     \underbrace{f_{\text{d}}(\rvx)}_{\text{real or fake}} + 
    \underbrace{\langle \rvh^\prime, f_{\text{s}}(\rvx) \rangle}_{\text{semantic alignment}}~,
\end{align}
Intuitively, $D(\rvx, \rvh^\prime)$ yields a high value for an image $\rvx$, when it is real (with large $f_{\text{d}}(\rvx)$ values) and the semantic similarity between $\rvh^\prime$ and $f_{\text{s}}(\rvx)$ is high. 
Similar ideas have been exploited in \cite{kang2020contragan,jeong2020training,zhang2021crossmodal}. Different from these methods, our model can utilize the pre-trained multi-modal feature space, which relieves the difficulty for discriminator in learning semantically meaningful features.


\begin{figure*}[ht!]
    \centering
   \includegraphics[width=0.8\linewidth]{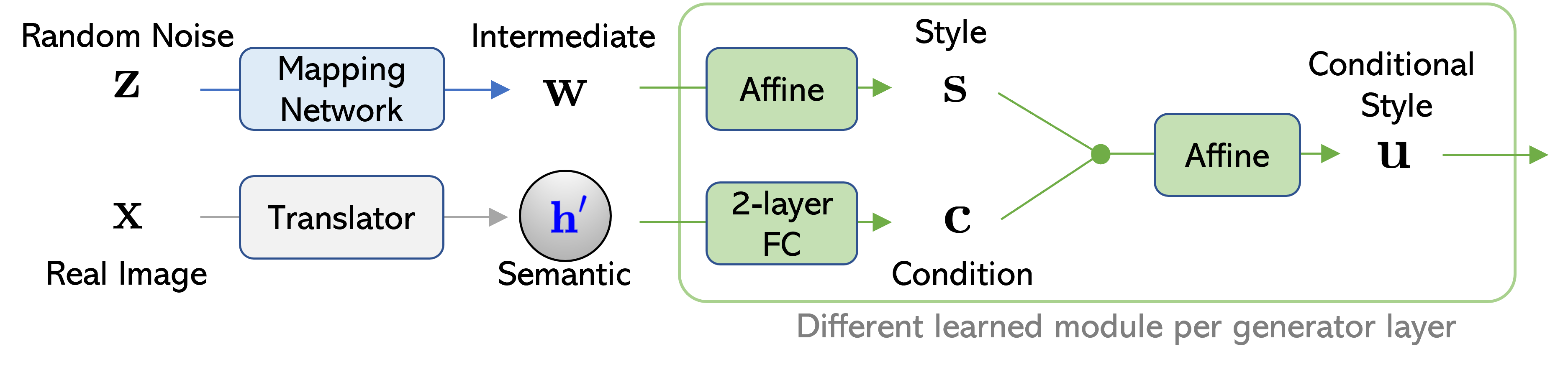}
   \vspace{-2mm}
    \caption{The process of injecting text-conditional information into each layer of the generator, where FC denotes fully-connected layer. The green modules have their own trainable parameters per generator layer. We can view the original StyleGAN2 constructs its StyleSpace as the process from $\rvz$ to $\rvs$. We propose to inject the semantic conditional information and further build our  Conditional StyleSpace, whose elements $\rvu$ will be used to modulate image generation. This figure illustrates the language-free setting, where real image is used to generate pseudo text feature $\rvh^\prime$; For the fully supervised text-to-image generation setting, real text is used for the extraction of text feature  $\rvh$. Please refer to the definition of translator in Section 3 for details.}
    \label{fig:generator}
\end{figure*}

\begin{figure*}[t!]
	\vspace{-2mm}\centering
	\begin{tabular}{c c c}
		\hspace{-3mm}
		\includegraphics[height=5cm]{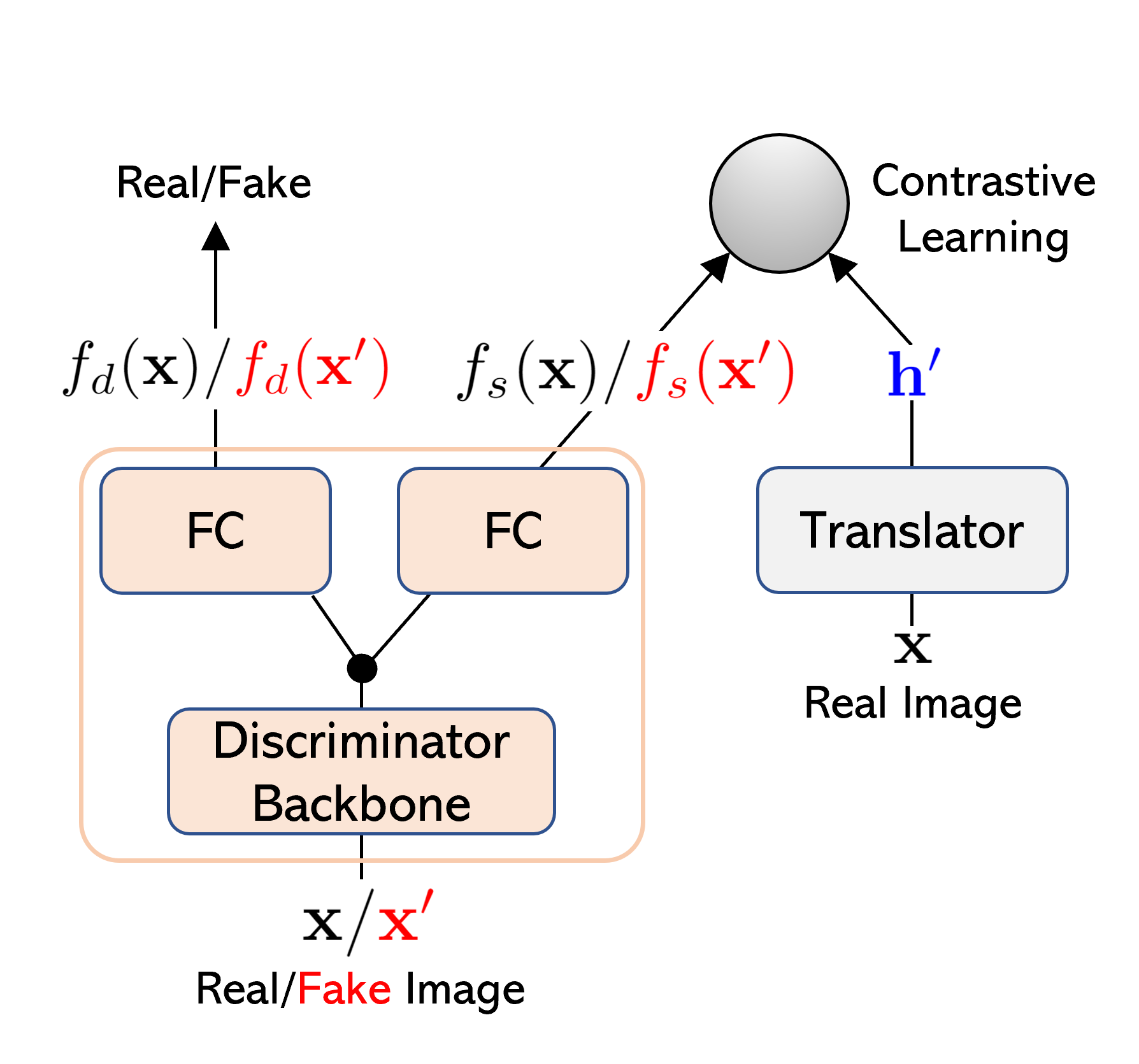}  & 
		\includegraphics[height=5cm]{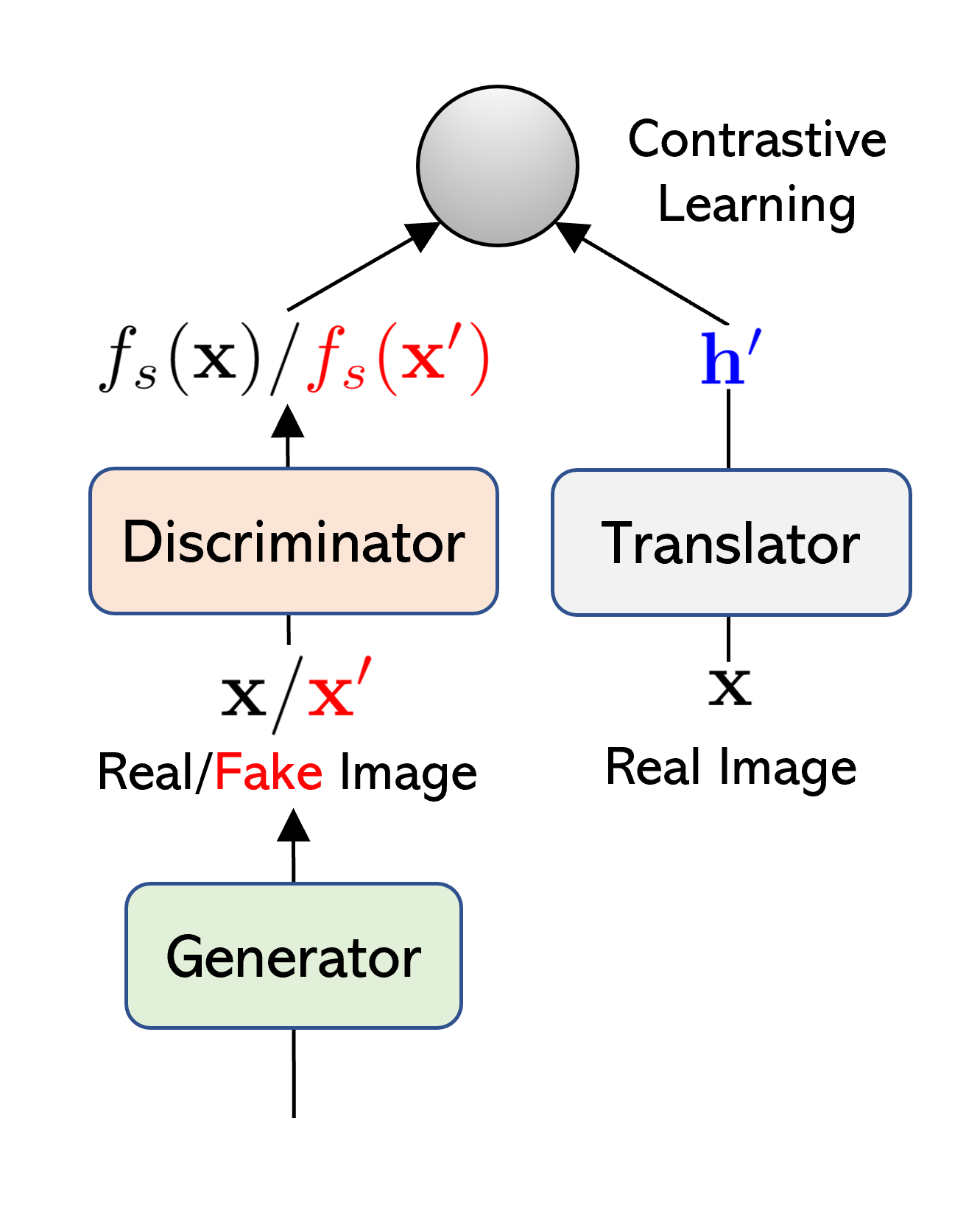}  &
		\includegraphics[height=5cm]{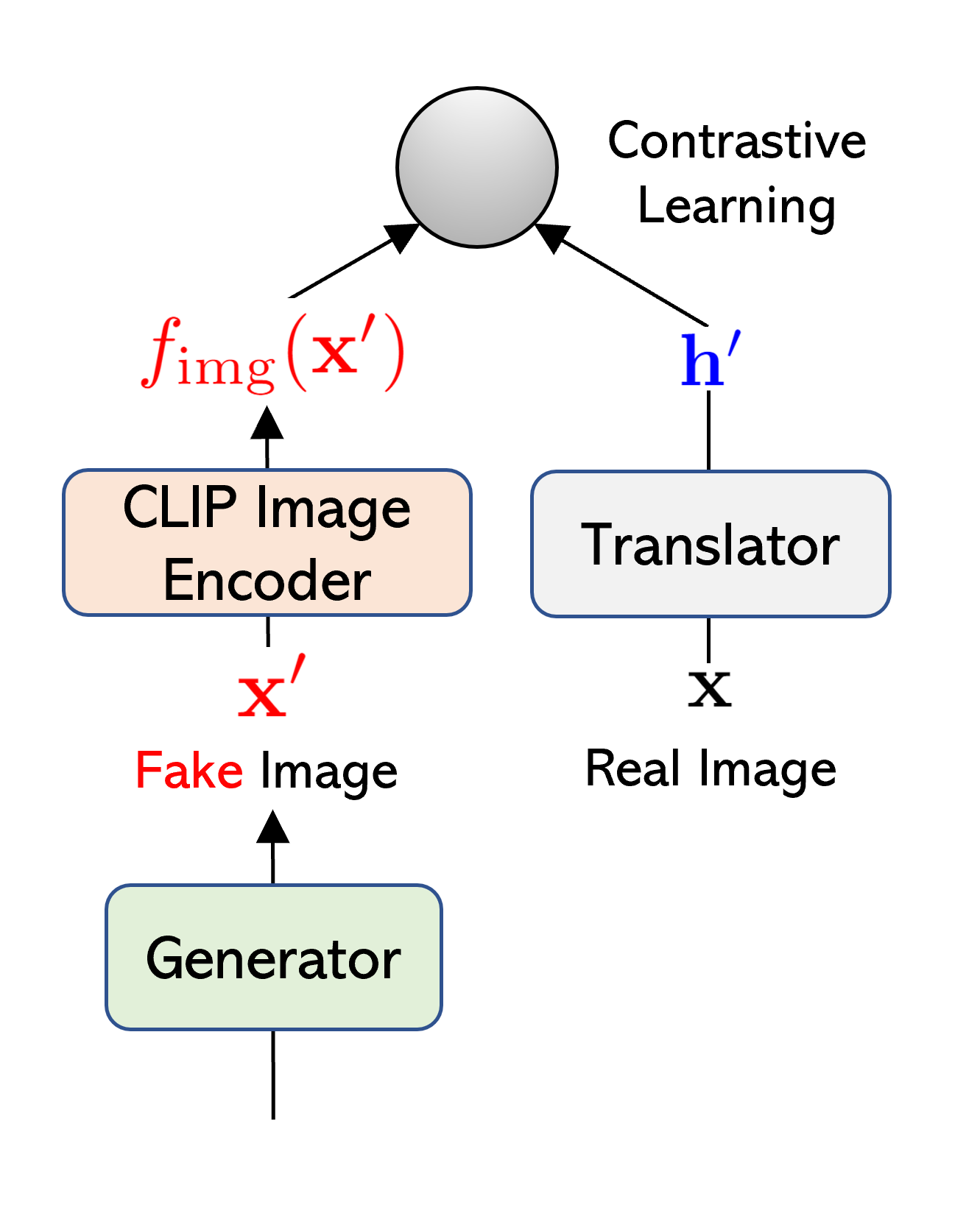}   \\
		(a) Discriminator output \vspace{0mm} & 
		(b) $\mathcal{L}_{\text{ConD}}$\hspace{-0mm} & 
		(c) $\mathcal{L}_{\text{ConG}}$ \hspace{0mm} 
	\end{tabular}
	\vspace{-2mm}
	\caption{Illustration of discriminator outputs and training objectives for the language-free setting.
	 }
	\label{fig:training_obj}
	\vspace{-1mm}
\end{figure*}
\subsection{Training Objectives}

For a mini-batch of $n$ images $\{\rvx_i\}_{i=1}^n$,
$\rvh^\prime _i$ is the corresponding generated pseudo text features for the $i$-th image.
Our model is trained in an adversarial manner, with additional contrastive losses to ensure that the GAN feature space is aligned with pre-trained CLIP. The first one is the standard conditional GAN loss. The losses for the generator and discriminator are defined, with the logits from \eqref{eq:logit}, as:
\begin{align}\label{eq:base_loss}
    \mathcal{L}_{\text{G}} &= - \sum_{i=1}^n \log \sigma(D(\rvx^\prime_i, \rvh_i^\prime)) ,  \\ 
    \mathcal{L}_{\text{D}} &= - \sum_{i=1}^n  \log \sigma(D(\rvx_i, \rvh_i^\prime)) - \sum_{i=1}^n  \log(1- \sigma(D(\rvx^\prime_i, \rvh_i^\prime))) \nonumber
\end{align}
where $\sigma(\cdot)$ denotes the Sigmoid function.

To enforce that the discriminator-extracted feature $f_{\text{s}}(\rvx)$ is semantically aligned in the pre-trained CLIP feature space, we consider the following contrastive regularizer for the discriminator:

\begin{align}\label{eq:d_contrastive_loss}
    \mathcal{L}_{\text{ConD}} = - \tau \sum_{i=1}^n \log \dfrac{\exp( \mathrm{Sim}(f_{\text{s}}(\rvx _i), \rvh_i^\prime)/\tau)}{\sum_{j=1}^n \exp( \mathrm{Sim}(f_{\text{s}}(\rvx _j), \rvh_i^\prime)/\tau)}, 
\end{align}
%
where $\mathrm{Sim}$ denotes the cosine similarity, $\tau$ is a non-negative hyper-parameter. Intuitively, $\mathcal{L}_{\text{ConD}} $ enforces the discriminator to output image feature $f_{\text{s}}(\rvx _i)$ that is similar to the corresponding text feature $\rvh^\prime_i$.

We further utilize the pre-trained CLIP model to improve the semantic correspondence of the generated images $\rvx_i^\prime$ and its conditioned pseudo text feature $\rvh ^\prime_i$. We define the following contrastive loss for the generator with the same hyper-parameter $\tau$ as \eqref{eq:d_contrastive_loss}:
\begin{align}\label{eq:clip_contrastive_loss}
   \hspace{-0.3cm}\mathcal{L}_{\text{ConG}} = - \tau \sum_{i=1}^n \log \dfrac{\exp(\mathrm{Sim}(f_{\text{img}}(\rvx^\prime _i), \rvh_i^\prime)/\tau)}{\sum_{j=1}^n \exp(\mathrm{Sim}(f_{\text{img}}(\rvx^\prime _j), \rvh_i^\prime)/\tau)}.
\end{align}

With the above contrastive regularizers, the final training loss for the generator and discriminator are defined as: 
\begin{align}
    \mathcal{L}_{\text{D}}^\prime & = \mathcal{L}_{\text{D}}  + \gamma \mathcal{L}_{\text{ConD}} \label{eq:final_loss_d} \\
    \mathcal{L}_{\text{G}}^\prime &  = \mathcal{L}_{\text{G}} +  \gamma \mathcal{L}_{\text{ConD}} + \lambda \mathcal{L}_{\text{ConG}}   \label{eq:final_loss_g}
\end{align}
where $\tau=0.5, \lambda =  \gamma = 10 $ for language-free settings, and $\tau=0.5, \lambda = 10 $,  $\gamma = 5 $ for fully-supervised settings\footnote{Details about hyper-parameter tuning are provided in the Appendix.}.

\subsection{Training Details}
\begin{algorithm}[t!]
    \caption{Language-free training of \shortname{}}\label{algo:lafite}
    \begin{algorithmic}[1]
        \STATE {\bfseries Input: An image dataset $\{\rvx_i\}_{i=1}^N$, pre-trained encoders $f_{\text{txt}}, f_{\text{img}}$, hyper-parameters $\tau>0$} 
    \WHILE { {\em not converge} }{
        \STATE Sample mini-batch $\{\rvx_i\}_{i=1}^n$;
        \STATE Sample perturbation noise $\{\mathbf{\epsilon}_i\}_{i=1}^n \sim \mathcal{N}(\mathbf{0}, \mathbf{I})$;
        \STATE {\color{blue} //  Pseudo text feature generation}
        \STATE Generate $\rvh^\prime _i$ according to \eqref{eq:adaptive_gaussian} or \eqref{eq:txt_fts_nn}; 
        
        \STATE {\color{blue} //  Forward pass of G and D}
        \STATE Sample latent noise $\{\mathbf{\rvz}_i\}_{i=1}^n \sim \mathcal{N}(\mathbf{0}, \mathbf{I})$;
        \STATE Synthesize fake image $\rvx_i^\prime$ with G using \eqref{eq:generator};
        \STATE Feed real/fake images to D using \eqref{eq:logit};
        \STATE {\color{blue} //  Update G and D with gradient descent}
        \STATE Update D with \eqref{eq:final_loss_d}; 
        \STATE Update G with \eqref{eq:final_loss_g}; 
        }
    \ENDWHILE
    \end{algorithmic}
\end{algorithm}

We summarize the language-free training schedule of \shortname{} in Algorithm~\ref{algo:lafite}. For the settings with full image-text pairs, one may replace pseudo text feature generation step with the ground-truth text feature $\rvh=f_{\text{txt}}(\rvt)$. 

\paragraph{Pre-training.}  To demonstrate the zero-shot task transfer ability of our model, we also consider a variant that is pre-trained on the Google Conceptual Captions 3M (CC3M) dataset \cite{Sharma2018ConceptualCA}, which consists of 3.3 millions of image-text pairs. For pseudo text-feature generation with trainable perturbation, we also train its inference model on CC3M. There is no image  overlapping between the pre-training and downstream datasets, which ensures the fairness when comparing our method against others in transfer learning. For face domain, we pre-trained a model on FFHQ dataset \cite{karras2019style} which contains 70,000 images. The pre-trained models can be fine-tuned with \shortname{} under language-free setting on different datasets, which will be discussed in next section.

\paragraph{Data augmentation.} In practice, we also consider image data augmentation to improve extracted image features $f_{\text{img}}(\rvx)$ in \eqref{eq:adaptive_gaussian}. We choose to use random cropping and avoid using augmentations like color transformation, because they may lead to mismatching between $\rvh^\prime$ and $\rvx$. The details are summarized in Appendix.


\section{Experiments}\label{sec:exp}

As the proposed \shortname{} is a versatile system, we conduct experiments under different settings, including the proposed language-free setting, as well as the zero-shot and fully-supervised text-to-image generation settings.
Due to the difference of two schemes to generate pseudo text features described in Section~\ref{sec:text_feature}, we denote our system in two variants: fixed perturbations as  \shortnameg{} and trainable perturbations as \shortnamen{}, respectively. 
All of our experiments are conducted on 4 Nvidia Tesla V100 GPUs, implemented using Pytorch \cite{paszke2019pytorch}. CLIP-ViT/B-32 is used in our methods unless specified. All the codes and pre-trained models will be publicly available upon acceptance.

\paragraph{Datasets.} We consider a suite of datasets that are commonly used in literature~\cite{xu2018attngan, zhu2019dm, zhang2021crossmodal, ye2021improving}, including MS-COCO \cite{cho2014learning}, CUB \cite{WahCUB_200_2011}, LN-COCO \cite{pont2020connecting}, Multi-modal CelebA-HQ (MM CelebA-HQ) \cite{xia2021tedigan}. All the images are scaled to resolution $256 \times 256$. Statistics of these datasets are summarized in Table \ref{tab:datasets} in the Appendix.
\paragraph{Evaluation metrics.}  Following \cite{ramesh2021zero, ding2021cogview}, we report the blurred Fr\'echet Inception Distance (FID) \cite{heusel2017gans} and Inception Score (IS) \cite{salimans2016improved} on MS-COCO dataset, which are computed using 30,000 generated images with randomly sampled text from validation set. FID-$k$ means the FID is computed after blurring all the images by a Gaussian filter with radius $k$.

\begin{figure}
    \centering
    \includegraphics[width=0.95\linewidth]{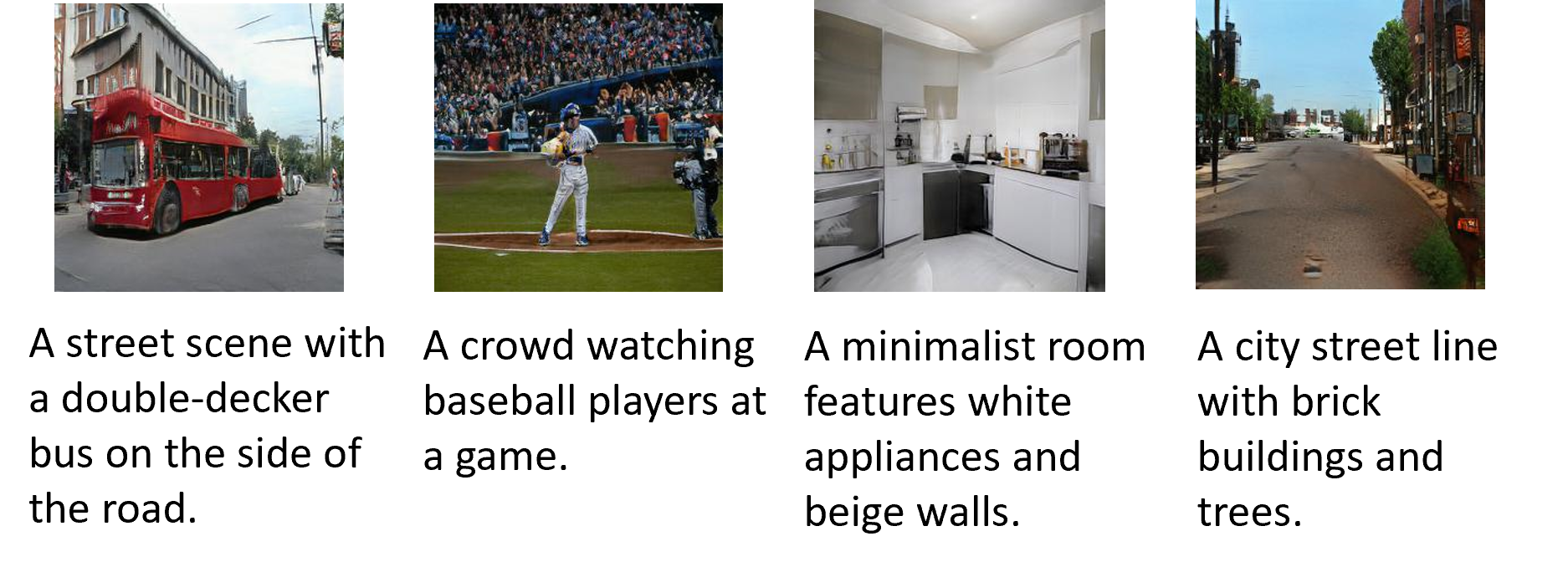}
    \caption{Language-free text-to-image generation examples on MS-COCO validation set. 
    }
	\vspace{-0.15in}
    \label{fig:generated}
\end{figure}

\begin{figure}
    \centering
    \includegraphics[width=\linewidth]{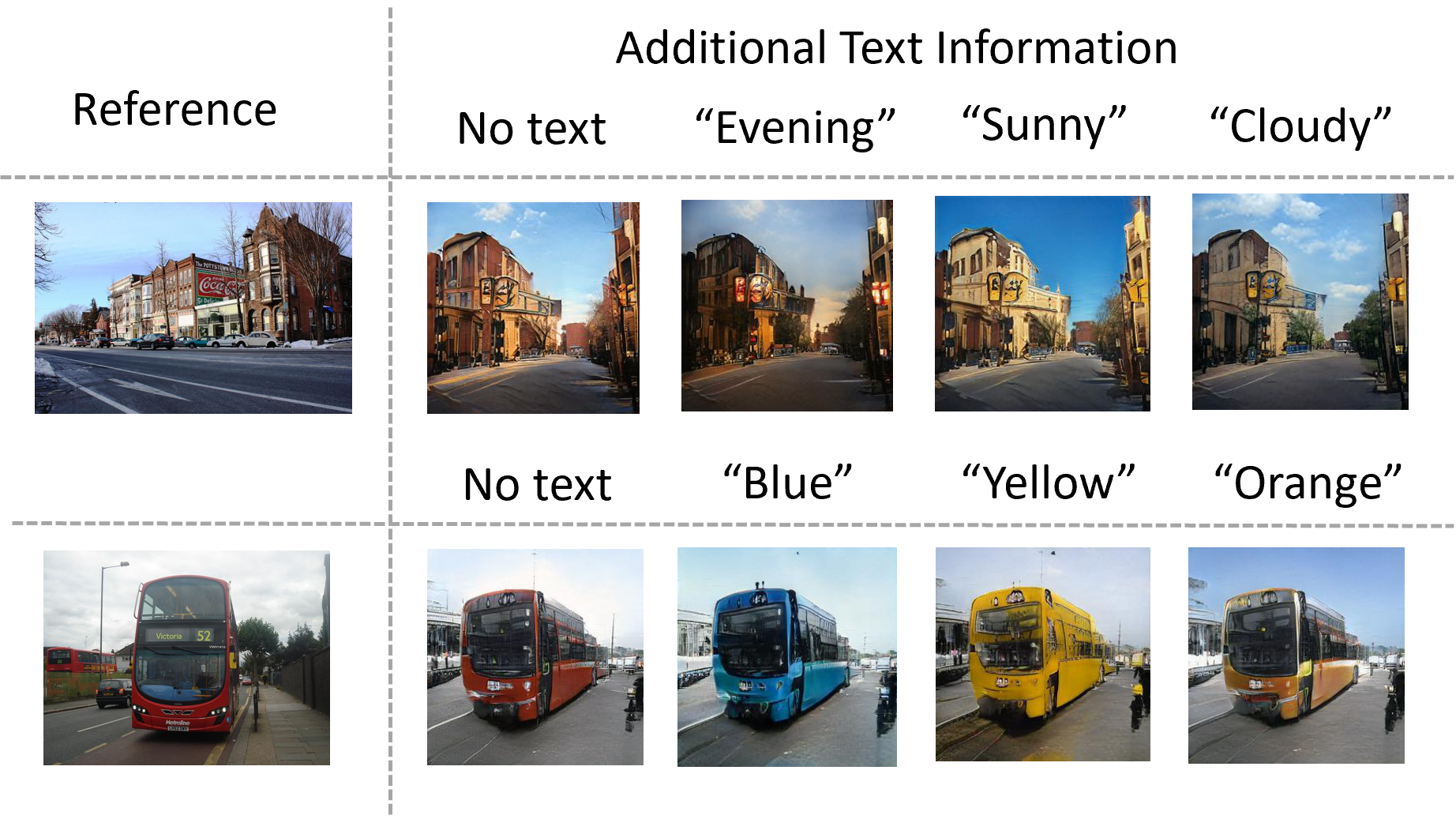}
	\vspace{-0.3in}
    \caption{Image generation with multi-modal conditions (conditioned on both image and text).}
    \label{fig:ms-coco_mixed_example}
\end{figure}

\subsection{Language-free Text-to-image Generation}
We first study \shortname{} under the proposed language-free setting, in which only images are provided in a given domain, and no paired caption is available during training.

\vspace{-0.5cm}
\paragraph{Captioning-based baseline:}
As a baseline, we employed the SoTA image captioning model VinVL \cite{zhang2021vinvl} to generate some associated captions for images. 
Note that MS-COCO image-text pairs were used to train the author-provided VinVL image captioning model, so the MS-COCO comparison is unfairly biased in favor of the baseline due to this information leakage. 
We compare this baseline method with our \shortname{} using the same network architecture and hyper-parameter setting for fairness. 
\begin{table}[t!]
\scalebox{0.76}{
    \begin{tabular}{lcccccc}
        \toprule
        Model & IS $\uparrow$ & FID-0  $\downarrow$ & FID-1  $\downarrow$ & FID-2  $\downarrow$ & FID-4  $\downarrow$ & FID-8  $\downarrow$\\
        \midrule
        Cap-Base & $15.83$ & $56.36$ & $54.99$ & $51.84$ & $44.81$ & $37.28$\\
        Cap-Large & $16.95$ & $47.21$ & $42.35$ & $37.85$ & $31.59$ &  $23.49$\\
        \shortnameg{} & $\mathbf{27.20}$ & $\mathbf{18.04}$ & $\mathbf{17.80}$ & $\mathbf{17.68}$ & $\mathbf{16.16}$ & $\mathbf{14.52}$ \\
        \shortnamen{} & $22.23$ & $26.56$ & $26.48$ & $25.82$ & $23.90$ & $19.27$\\
        \bottomrule
    \end{tabular}
    }
    \centering
    \caption{Results of language-free setting on MS-COCO dataset. `Cap' indicates a text-to-image generation baseline method based on VinVL captioning. 
    }    
    \label{tab:language_free_on_coco}
\end{table}
The main results are in Table \ref{tab:language_free_on_coco}. Both variants of our \shortname{} significantly outperform the captioning-based baseline method. The simple \shortnameg{} performs the best on this dataset, indicating the generality of the method. For \shortnamen{}, note that CC3M is used to train the inference model, thus there is no information leakage in \shortnamen{} method as we test \shortnamen{} on the MS-COCO dataset. Some generated examples are provided in Figure \ref{fig:generated}, from which we can see that our \shortname{} leads to text-aligned generation though no text data is used during training, verifying the effectiveness of the proposed method.

\begin{table*}[ht!]
\scalebox{0.9}{
    \begin{tabular}{l|cccccccc}
        \toprule
        Model & IS $\uparrow$ & FID-0  $\downarrow$ & FID-1  $\downarrow$ & FID-2  $\downarrow$ & FID-4  $\downarrow$ & FID-8  $\downarrow$ & SOA-C $\uparrow$ & SOA-I $\uparrow$\\
        \midrule
        DALL-E  & $17.90$ & $27.50$ & $28.00$ & $45.50$ & $83.50$ & $85.00$ & - & -\\
        CogView & $18.20$ & $27.10$ & $19.40$ & $13.90$ & $19.40$ & $23.60$ & - & -\\
        \shortname{}\lafetilogo{} & $\mathbf{26.02}$ & $\mathbf{26.94}$ & $22.97$ & $18.70$ & $\mathbf{15.72}$ & $\mathbf{14.79}$ & $37.37$ & $54.25$ \\
        \bottomrule
    \end{tabular}
    \centering
    \caption{Results of zero-shot setting on MS-COCO dataset, the model is pre-trained with image-text pairs from CC3M dataset.
    }    
    \label{tab:zero_shot_on_coco}
    }
\end{table*}

\begin{table*}[ht!]
\scalebox{0.9}{
    \begin{tabular}{l|cccc|cc|cc|cc}
        \toprule
        & \multicolumn{4}{c}{MS-COCO} & \multicolumn{2}{|c}{CUB} & \multicolumn{2}{|c}{LN-COCO}& \multicolumn{2}{|c}{MM CelebA-HQ}\\
        Model & IS $\uparrow$ & FID $\downarrow$ & SOA-C  $\uparrow$ & SOA-I  $\uparrow$ & IS $\uparrow$ & FID $\downarrow$ & IS $\uparrow$ & FID $\downarrow$ & IS $\uparrow$ & FID $\downarrow$ \\
        \midrule
        AttnGAN  & $23.61$ & $33.10$ &$25.88$ & $39.01$ & $4.36$ & $23.98$ & $20.80$ & $51.80$ & - & 125.98 \\
        Obj-GAN & $24.09$ & $36.52$ &$27.14$ & $41.24$ & - & - & - & - & - & -\\
        DM-GAN & $32.32$ & $27.34$ &$33.44$ & $48.03$& $4.75$ & $16.09$ & - & - & - & 131.05\\
        OP-GAN & $27.88$ & $24.70$ & $35.85$ & $50.47$ &- & - & - & - & - & -\\
        DF-GAN  & - & $21.42$  &-&-& $5.10$ & $14.81$ & - & - & - & 137.60\\ 
        XMC-GAN  & $30.45$ & $9.33$ & $50.94$ & $71.33$&- & - & $28.37$ & $14.12$ & - & - \\
        \midrule
        \shortname{}\lafetilogo{} & $\mathbf{32.34}$ & $\mathbf{8.12}$ & $\mathbf{61.09}$ & $\mathbf{74.78}$& $\mathbf{5.97}$ & $\mathbf{10.48}$ & $26.32$ & $\mathbf{11.78}$ & $\mathbf{2.93}$ & $\mathbf{12.54}$    \\
        \bottomrule
    \end{tabular}
    \centering
    \caption{Standard text-to-image generation on CUB, LN-COCO and MM CelebA-HQ datasets.
    }    
    \label{tab:t2i_main_results}
    }
\end{table*}

Furthermore, we can actually perform generation conditioned on images: For a given image, we generate an image-conditioned pseudo text feature vector with 
\shortname{}. Passing this pseudo text feature vector to $G$ leads to generated images that are similar to the given image. 
Consequently, \shortname{} enables image generation with multi-modal conditions, {\em i.e.} it can be conditioned on both image and text simultaneously. The implementation details are discussed in the Appendix.
Some generated examples are provided in Figure \ref{fig:ms-coco_mixed_example}, more results are provided in the Appendix.

\subsection{Zero-Shot Text-to-image Generation}
Zero-shot is a setting to evaluate a pre-trained text-to-image generation model, without training the model on any of downstream data. MS-COCO dataset is used for evaluating our model pre-trained on CC3M. 
%
The main results are shown in Table \ref{tab:zero_shot_on_coco}.
Compared to DALL-E \cite{ramesh2021zero} and CogView \cite{ding2021cogview}, \shortname{} achieves better quantitative results in most cases. We also emphasize that our model has only 75 millions of trainable parameters, while DALL-E has over 12 billions of parameters.
Arguably, our pre-training dataset CC3M is much smaller\footnote{Though we acknowledge that \shortname{} is based on an off-the-shelf discriminate model CLIP, which is trained on 400 million image-text pairs}, compared to the pre-training dataset used in DALL-E, which contains 250 millions of image-text pairs.

\subsection{Standard Text-to-image Generation}

We now consider the standard text-to-image generation task, where all the ground-truth image-text pairs are provided during training. 
We compare \shortname{} against a series of competitive systems: AttnGAN \cite{xu2018attngan}, Obj-GAN \cite{li2019object}, DM-GAN \cite{zhu2019dm}, OP-GAN \cite{9184960}, DF-GAN \cite{tao2021dfgan} and XMC-GAN \cite{zhang2021crossmodal}. The main results evaluated by FID and IS on different datasets are provided in Table \ref{tab:t2i_main_results}. We also report the Semantic Object Accuracy (SOA) on MS-COCO following previous works \cite{9184960, zhang2021crossmodal}.
Results of competitive models are directly cited from the corresponding papers. 
It is clear that our proposed model consistently outperforms all other methods, creating new SoTA results in standard text-to-image generation.

\begin{table}[t!]
\scalebox{0.66}{
    \begin{tabular}{l|cc|cc|cc|cc}
        \toprule
         & \multicolumn{2}{c}{MS-COCO}  & \multicolumn{2}{|c}{CUB}& \multicolumn{2}{|c}{LN-COCO} & \multicolumn{2}{|c}{MM CelebA-HQ}\\
         Methods & IS $\uparrow$ & FID $\downarrow$ & IS $\uparrow$ & FID $\downarrow$ & IS $\uparrow$ & FID $\downarrow$ & IS $\uparrow$ & FID $\downarrow$ \\
         \midrule
          &\multicolumn{8}{c}{Training from Scratch} \\
         \shortnameg{} & $\mathbf{27.20}$ & $\mathbf{18.04}$ &$\mathbf{4.32 }$& $\mathbf{27.53}$ & $\mathbf{18.49}$ & $38.95$ & $2.78$ & $\mathbf{32.75}$  \\
         \shortnamen{} & $22.23$ & $26.56$ & $4.06$ & $46.32$ & $18.17$ & $\mathbf{36.19}$ & $\mathbf{2.89}$ & $50.34$ \\
         \midrule
         &\multicolumn{8}{c}{Fine-tuned from Pre-trained Model} \\
         \shortnameg{} & $24.89$ & $20.89$ & $\mathbf{6.13}$ & $\mathbf{35.99}$ & $19.32$ & $34.96$& $3.10$ & $\mathbf{15.74}$ \\
         \shortnamen{} & $\mathbf{26.55}$ & $\mathbf{17.44}$ & $4.36$ & $37.91$  &$\mathbf{20.02 }$& $\mathbf{33.76}$ & $\mathbf{3.19}$ & $29.42$\\
         \bottomrule
    \end{tabular}
    \centering
    \caption{Comparisons between two schemes for language-free training on different datasets.}    
    \label{tab:baselines}
	\vspace{-0.1in}
    }
\end{table}

\subsection{Adaptation of Pre-trained Models }

\paragraph{Language-free model fine-tuning.} Compared with existing works, one key advantage of the {\em pre-trained} \shortname{} model is that it naturally enables language-free model fine-tuning. The results are provided in Table \ref{tab:baselines}, where both \shortnameg{} and \shortnamen{} are investigated on different datasets. We see that fine-tuning from the pre-trained model generally outperform training from scratch. 
We also notice that performance of pre-trained \shortname{} largely depends on the domain gap in pre-training and fine-tuning datasets.
For example, \shortnamen{} sometimes obtains worse results than \shortnameg{}, especially when the fine-tuning dataset is dissimilar to CC3M, \ie, CUB and MM CelebA-HQ.
This indicates that the inference model used for generating text features may have biases, because it may over-fit to its training dataset CC3M.  

Pre-trained \shortname{}  is also highly training-efficient. For example, training from scratch with \shortname{} on MS-COCO dataset requires around 4 days to reach FID of 18, while fine-tuning only needs 3 hours. This becomes a critical advantage especially when we require several text-to-image generation models across different datasets.  

\paragraph{Semi-supervised fine-tuning.} Adaptation of pre-trained \shortname{} is sample-efficient. One interesting question is, how much percentage of image-text pairs do we need to outperform previous SoTA XMC-GAN on MS-COCO dataset? To answer this question, we conduct experiment in which only a portion of the images are associated with ground-truth text. Our model is first pre-trained using all the images under the language-free setting, then it is fine-tuned with varying percentages of image-text pairs.
The main results are summarized in Figure \ref{fig:curve_ms_coco}. Our method outperforms XMC-GAN on both IS and FID when less than half of total of the image-text pairs are employed. 

\begin{figure}[t!]
      \begin{subfigure}{0.47\linewidth}
        \includegraphics[width=1.\linewidth]{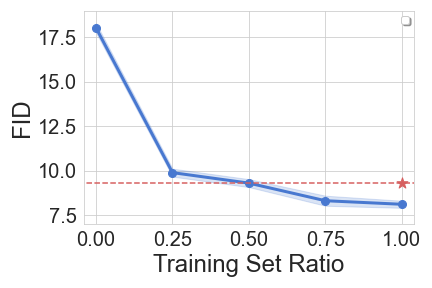}
        \caption{FID $(\downarrow)$}
        \label{fig:curve_fid}
      \end{subfigure}
      \begin{subfigure}{0.47\linewidth}
        \includegraphics[width=1.\linewidth]{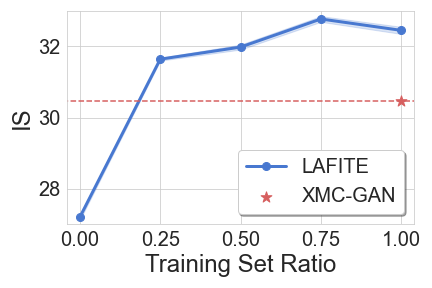}
        \caption{IS $(\uparrow)$}
        \label{fig:curve_is}
      \end{subfigure}
    \centering
    \caption{Comparison of \shortname{} and prior art XMC-GAN. X-axis is the percentage of image-text pairs in the full MS-COCO dataset. XMC-GAN has over 166 millions trainable parameters, while our \shortname{} only has 75 millions trainable parameters.}
    \label{fig:curve_ms_coco}
	\vspace{-0.1in}
\end{figure}


\subsection{Ablation Study}
\paragraph{Ablation study of training objectives}
We first investigate the impact of each component in our objective functions. The standard generator and discriminator losses are always employed, we ablate by excluding $\mathcal{L}_{\text{ConG}}$ and $\mathcal{L}_{\text{ConD}}$ one by one. The results are provided in Table \ref{tab:ablation_study_loss}. For both variants of \shortname{}, it is observed the model performance could drop significantly.


\begin{table}[t!]
\scalebox{0.8}{
    \begin{tabular}{c|cc|cccc}
        \toprule
          Model & $\mathcal{L}_{\text{ConG}}$ &  $\mathcal{L}_{\text{ConD}}$& IS $\uparrow$ & FID $\downarrow$ & SOA-C  $\uparrow$ & SOA-I  $\uparrow$\\
         \midrule
             \rotatebox{90}{\multirow{2}{*}{\hspace{-13mm} \bf \shortnameg{} }}
            &
             &  & $14.79$ & $33.03$& $9.64$ & $18.40$\\
            & \checkmark & & $17.78$ & $29.65$ & $16.53$ & $30.33$\\ 
            & &\checkmark & $22.28$ & $21.25$ & $29.09$ & $43.77$\\ 
             & \checkmark & \checkmark & $\mathbf{27.20}$ & $\mathbf{18.04}$ & $\mathbf{36.84}$ & $\mathbf{54.16}$\\
            \midrule
            \rotatebox{90}{\multirow{2}{*}{\hspace{-13mm} \bf \shortnamen{} }}
            & & & $11.05$ & $72.03$& $8.28$ & $14.46$\\
             &  \checkmark &  & $20.02$ & $30.67$ & $26.60$ & $41.26$  \\ 
             &   & \checkmark & $19.14$ & $33.88$ & $33.32$ & $49.86$\\ 
            &   \checkmark & \checkmark & $\mathbf{22.23}$ & $ \mathbf{26.48}$  & $\mathbf{36.86}$ & $\mathbf{54.02}$\\
         \bottomrule
    \end{tabular}
    }
    \centering
    \caption{Ablations of training losses on MS-COCO dataset, $\checkmark$ means the component is used during training.}
    \label{tab:ablation_study_loss}
    \vspace{-0.1in}
\end{table}

\begin{table}[t!]
\scalebox{0.76}{
    \begin{tabular}{@{\hspace{-0pt}}l|@{\hspace{8pt}}c@{\hspace{2pt}}@{\hspace{5pt}}c@{\hspace{5pt}}c@{\hspace{5pt}}c@{\hspace{5pt}}c}
        \toprule
         Model & Feature dim & IS $\uparrow$ & FID $\downarrow$ & SOA-C  $\uparrow$ & SOA-I  $\uparrow$ \\
         \midrule
            RoBERTa-Base &  $768$ & $15.95$ & $29.55$ & $11.58$ & $22.89$\\
            RoBERTa-Large & $1024$ & $14.11$ & $35.77$ & $7.72$ & $16.03$\\
            CLIP(B-32) Text encoder & $512$ & $24.54$ & $16.21$ & $47.74$ & $61.86$\\
            CLIP(B-16) Text encoder & $512$ & $24.90$ & $15.97$ & $47.80$ & $62.71$\\
            CLIP(B-32) & $512$ & $31.88$ & $8.62$ & $59.51$ & $73.76$\\
            CLIP(B-16) & $512$ & $\mathbf{32.34}$ & $\mathbf{8.12}$ & $\mathbf{61.09}$ & $\mathbf{74.78}$ \\
         \bottomrule
    \end{tabular}
    }
    \centering
    \caption{Results of using different pre-trained models on MS-COCO dataset.}    
    \label{tab:ablation_study_roberta}
    \vspace{-0.1in}
\end{table}

\paragraph{Ablations of pre-trained text/image encoders}
To demonstrate the importance of using a multi-modal feature-aligned pre-trained model in our \shortname{}, we  compare the CLIP model and other single-modality models. 
We adopt the popular RoBERTa \cite{liu2019roberta} as the baseline text encoder, which was  trained on a large text corpus only. 
Note that it is infeasible to perform language-free training without the joint feature space. Thus this experiment is based on fully-supervised text-to-image generation setting. 
For a fair comparison, we also report the results of only using the text encoder of CLIP while discarding the image encoder. 
In this setting, there is no image encoder thus the $\mathcal{L}_{\text{ConG}}$ term is removed from the objective function consequently.  
The results are reported in Table \ref{tab:ablation_study_roberta}. As expected, even if the image encoder of CLIP is not used, models with only CLIP text encoder still significantly outperform models using RoBERTa. From the results, we can conclude that: $(i)$ The feature space of CLIP is semantically meaningful for text-to-image generation, thus only  using text encoder of CLIP still leads to better results than RoBERTa; $(ii)$ Text-to-image generation results can be improved by using a feature-aligned joint feature space (CLIP vs others),
and can be further improved with a stronger joint space (CLIP-ViT/B-16 outperforms CLIP-ViT/B-32, where ViT/B-16 and ViT/B-32 are different designs of visual transformers \cite{dosovitskiy2020image}). 

\section{Conclusion}
We have presented \shortname{}, an approach to build text-to-image generation systems without domain-specific image-text pairs in training. We achieve the goal by resorting to generating pseudo text features from images. 
Excellent performance  in a variety of text-to-image generations tasks have demonstrated the effectiveness of \shortname{}, including language-free, zero-shot and fully supervised settings. In particular, \shortname{} creates new SoTA in zero-shot setting, with only 1\% trainable parameter counts compared with recent advances such as DALL-E/CogView. \shortname{} also outperforms prior arts in the fully-supervised settings.
We believe that language-free training is a promising direction to enable broader application areas for text-to-image generation, as it significantly lowers the burden on data collection. One interesting future direction is to explore image synthesis in the wild, where long tail and open set conditions are provided for generation.

{\small
\bibliographystyle{ieee_fullname}
\bibliography{egbib}
}

\clearpage

\appendix
\section{Appendix}

\subsection{Theoretical Results}
\begin{theorem}\label{thm:threshold}
For a given threshold $c>0$, the generated text feature by \shortnameg{} satisfies $\text{Sim}(f_{\text{img}}(\rvx_i), \rvh^\prime_i) \geq c$ with probability at least
\begin{align*}
   &\text{Prob}(\text{Sim}(f_{\text{img}}(\rvx_i), \rvh^\prime_i) \geq c) \\= & 1 - \int_{-1}^{(c-1)/\xi+c} \dfrac{\Gamma(d/2+1/2)}{\sqrt{\pi} \Gamma(d/2)} (1-x^2)^{d/2-1} \text{d}x 
\end{align*}
where $d$ is the dimension number of features, $\Gamma(z) = \int_0^{\infty}x^{z-1}e^{-x}\text{d}x$ is the Gamma function.
\end{theorem}
\begin{proof}
Without loss of generality, we omit the subscript for clearness.
\begin{align*}
    & \text{Sim}(f_{\text{img}}(\rvx), \rvh^\prime) \\
    = & \dfrac{\langle f_{\text{img}}(\rvx), \rvh^\prime \rangle }{\Vert  f_{\text{img}}(\rvx)\Vert_2 \Vert \rvh^\prime \Vert_2}\\
    = & \dfrac{\langle f_{\text{img}}(\rvx), f_{\text{img}}(\rvx) + \xi \epsilon \Vert f_{\text{img}}(\rvx) \Vert_2 / \Vert \epsilon \Vert_2 \rangle }{\Vert  f_{\text{img}}(\rvx)\Vert_2 \Big\Vert f_{\text{img}}(\rvx) + \xi \epsilon \Vert f_{\text{img}}(\rvx) \Vert_2 / \Vert \epsilon \Vert_2 \Big \Vert_2} \\
    = & \dfrac{\Vert f_{\text{img}}(\rvx)\Vert^2 + \xi \epsilon^\intercal f_{\text{img}}(\rvx)  \Vert f_{\text{img}}(\rvx) \Vert_2 / \Vert \epsilon \Vert_2 }{\Vert  f_{\text{img}}(\rvx)\Vert_2 \Big\Vert f_{\text{img}}(\rvx) + \xi \epsilon \Vert f_{\text{img}}(\rvx) \Vert_2 / \Vert \epsilon \Vert_2 \Big\Vert_2}
\end{align*}
Denote $a = f_{\text{img}}(\rvx)/ \Vert f_{\text{img}}(\rvx) \Vert_2, b = \epsilon / \Vert \epsilon \Vert_2$, then we have
\begin{align*}
    & \text{Sim}(f_{\text{img}}(\rvx), \rvh^\prime) \\
    = & \dfrac{\Vert f_{\text{img}}(\rvx)\Vert^2 + \xi \epsilon^\intercal f_{\text{img}}(\rvx)  \Vert f_{\text{img}}(\rvx) \Vert_2 / \Vert \epsilon \Vert_2 }{\Vert  f_{\text{img}}(\rvx)\Vert_2 \Big\Vert f_{\text{img}}(\rvx) + \xi \epsilon \Vert f_{\text{img}}(\rvx) \Vert_2 / \Vert \epsilon \Vert_2 \Big\Vert_2}\\
    = & \dfrac{1 + \xi a^\intercal b }{\Vert a  + \xi b \Vert _2}\\
    \geq & \dfrac{1 + \xi a^\intercal b }{\Vert a \Vert_2 + \xi \Vert b \Vert_2 }\\
     = & \dfrac{1 + \xi a^\intercal b}{1 + \xi}
\end{align*}
Consequently,
\begin{align*}
    & \text{Prob}(\text{Sim}(f_{\text{img}}(\rvx_i), \rvh^\prime_i) \geq c) \\
    \geq & \text{Prob}(\dfrac{1 + \xi a^\intercal b}{1 + \xi} \geq c) \\
     = & \text{Prob}(1 + \xi a^\intercal b \geq c + c \xi) \\
      = & \text{Prob}(a^\intercal b \geq (c - 1 +  c \xi)/\xi) 
\end{align*}

By the cumulative distribution function (CDF) of inner product of random vectors on sphere \cite{cho2009inner}, we know that
\begin{align*}
     \text{Prob}(a^\intercal b \leq z) = \int_{-1}^z  \dfrac{\Gamma(d/2 + 1/2)}{\sqrt{\pi} \Gamma(d/2)}(1-x^2)^{d/2 - 1} \text{d}x
\end{align*}
where $d$ is the dimension number of features, $\Gamma(z) = \int_0^{\infty}x^{z-1}e^{-x}\text{d}x$ is the Gamma function. Thus we have
\begin{align*}
    & \text{Prob}(\text{Sim}(f_{\text{img}}(\rvx_i), \rvh^\prime_i) \geq c) \\
    \geq & \text{Prob}(a^\intercal b \geq (c - 1 +  c \xi)/\xi) \\
     = & 1 - \int_{-1}^{(c-1)/\xi+c} \dfrac{\Gamma(d/2+1/2)}{\sqrt{\pi} \Gamma(d/2)} (1-x^2)^{d/2-1} \text{d}x, 
\end{align*}
which completes the proof.
\end{proof}

\subsection{Experiment Details}

\paragraph{Datasets} The statistics of datasets are summarized in Table \ref{tab:datasets}.

\begin{algorithm}[ht!]
    \caption{Image feature extraction process}\label{algo:algorithm_image_feature}
    \begin{algorithmic}[1]
        \STATE {\bfseries Input: An image dataset $\{\rvx_i\}_{i=1}^N$, image resolution $w \times w$, pre-trained $f_{\text{img}}$, hyper-parameters $a>0, k\geq 1$} 
        \STATE {\color{blue} //  Image feature generation}
        \FOR{$i=1$ {\bfseries to} $n$}
        \IF{use data augmentation}{
            \STATE Initialize $\rvh^\prime _i \leftarrow \mathbf{0}$;
            \FOR{$j=1$ {\bfseries to} k}
                \STATE $\rvh^\prime _i \leftarrow \rvh^\prime _i + f_{\text{img}}(\text{CROP}(\rvx_i))$, where $\text{CROP}(\cdot)$ denotes randomly cropping image to be $w^\prime\times w^\prime$, $w^\prime$ is an integer randomly sampled from the range $\left[a,  w\right]$;
            \ENDFOR
            \STATE $\rvh^\prime _i \leftarrow \rvh^\prime _i/k$;}
        \ELSE
            \STATE Initialize $\rvh^\prime _i \leftarrow f_{\text{img}}(\rvx_i)$;
        \ENDIF
        \ENDFOR
    \end{algorithmic}
\end{algorithm}
\begin{table}[t!]
\scalebox{0.96}{
    \begin{tabular}{lccc}
        \toprule
          Dataset& \#train & \#validation & caption/image\\
         \midrule
         MS-COCO & 82k &40k & 5 \\
         CUB & 9k &3k & 10 \\
         LN-COCO & 134k & 8k & 1 \\
         MM CelebA-HQ & 24k & 6k & 10 \\
         \bottomrule
    \end{tabular}
    }
    \centering
    \caption{Statistics of datasets. The last column indicates ratio of captions vs images. 
    }    
    \label{tab:datasets}
\end{table}

\paragraph{Image feature extraction} In practice, we use random cropping as data augmentation when we extract the image features, which is presented in Algorithm \ref{algo:algorithm_image_feature}. The pseudo text features will be generated by perturb the average feature of augmented samples. In our implementation, we set $k=1, a=256$ to extract image features used in generating $\rvh ^\prime$, while we set $k=1, a=128$ in contrastive loss \eqref{eq:clip_contrastive_loss}.

\paragraph{Hyper-parameter} The hyper-parameters are selected based on the performance on MS-COCO dataset. Specifically, $\tau$ is selected from $\left[0.1, 0.2, 0.5, 1.0, 2.0 \right]$, $\lambda, \gamma$ are selected from $\left[0, 1, 2, 5, 10, 20, 50 \right]$. 

\paragraph{Exponential sharpening} In practice, we found that applying an extra exponential sharpening in contrastive loss makes it easier to reproduce the experiment results, i.e. we add an extra exponential operation right before the softmax function in \eqref{eq:d_contrastive_loss} and \eqref{eq:clip_contrastive_loss}. Our implementation can be found at \url{https://github.com/drboog/Lafite}.

\subsection{More Results}

We provide the implementation details of image generation with multi-modal conditions, an ablation study on discriminator, and more generated examples under language-free setting. 

\paragraph{Generation with multi-modal conditions}
To generate an image conditioned on both a reference image and text description, we first extract the text feature $\rvh_1$ from the given text, and pseudo text feature $\rvh^\prime_2$ from the image. Then $\rvh_1, \rvh^\prime_2$ will be feed into the pre-trained generator, leading to two conditional style codes $\rvu_1$ and $\rvu_2$. We construct a new conditional style code, whose elements are randomly selected from the corresponding elements in either $\rvu_1$ or $\rvu_2$. The new conditional style code will be fed into the generator to generate the desired image.

Note that generation conditioned on image is not reconstruction. Thus when only a reference image is provided, the generated image may have differences with the given image. However, they will share some visible characteristics that are semantic meaningful as illustrated in our examples.

\paragraph{Ablation study on discriminator}
We test the impact of each term of \ref{eq:logit} under language-free setting. The results are provided in Table \ref{tab:logit_ablation}, from which we can see that both terms are important, while the "real or fake" term seems to be more important.
\begin{table}[t!]
    \centering
    \begin{tabular}{cccccc}
    \toprule
    {\color{red}RoF} & {\color{red}SA} & FID $\downarrow$ & IS $\uparrow$ & SOA-C $\uparrow$ & SOA-I $\uparrow$\\
    \midrule
     \checkmark & & $24.85$ & $23.74$ & $ 30.54$ & $48.72$\\
         & \checkmark & $25.42$ & $21.14$ & $23.14$ & $38.32$\\ 
         \checkmark & \checkmark & $18.04$ & $27.20$ & $36.84$ & $54.16$\\
         \bottomrule
    \end{tabular}
    \caption{Ablation study on discriminator logits in language-free setting. {\color{red}RoF} denotes ``real or fake'' term, {\color{red}SA} denotes ``semantic alignment'' term.}
    \label{tab:logit_ablation}
\end{table}

\paragraph{Generated examples}
Some text-to-image generation results on CUB, MS-COCO, MM CelebA-HQ, LN-COCO are provided in the following figures.

\begin{figure*}[ht!]
    \centering
    \includegraphics[width=0.7\linewidth]{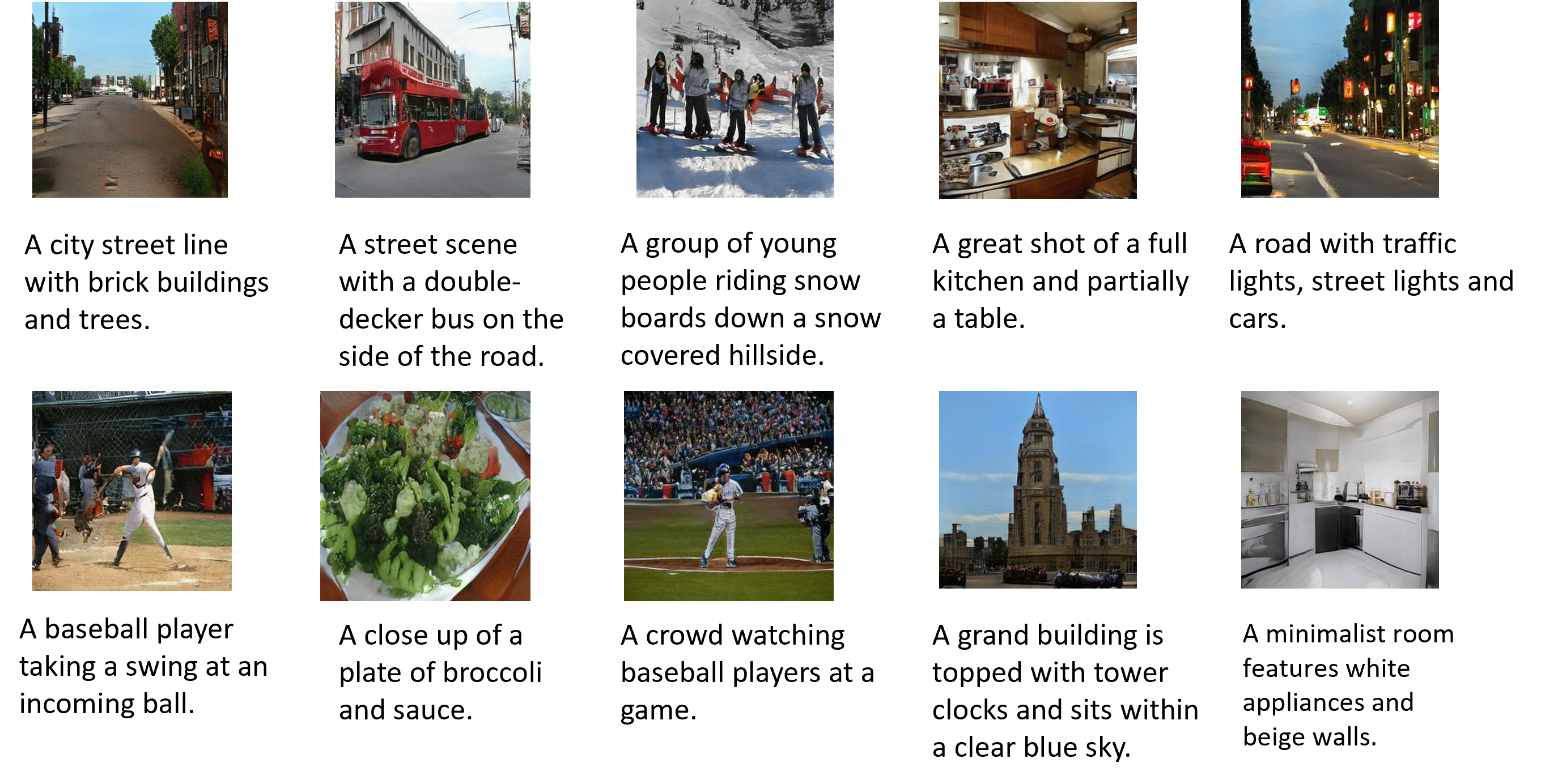}
    \caption{Generating examples on MS-COCO dataset.}
    \label{fig:ms-coco_generated}
\end{figure*}

\begin{figure*}[h!]
    \centering
    \includegraphics[width=0.7\linewidth]{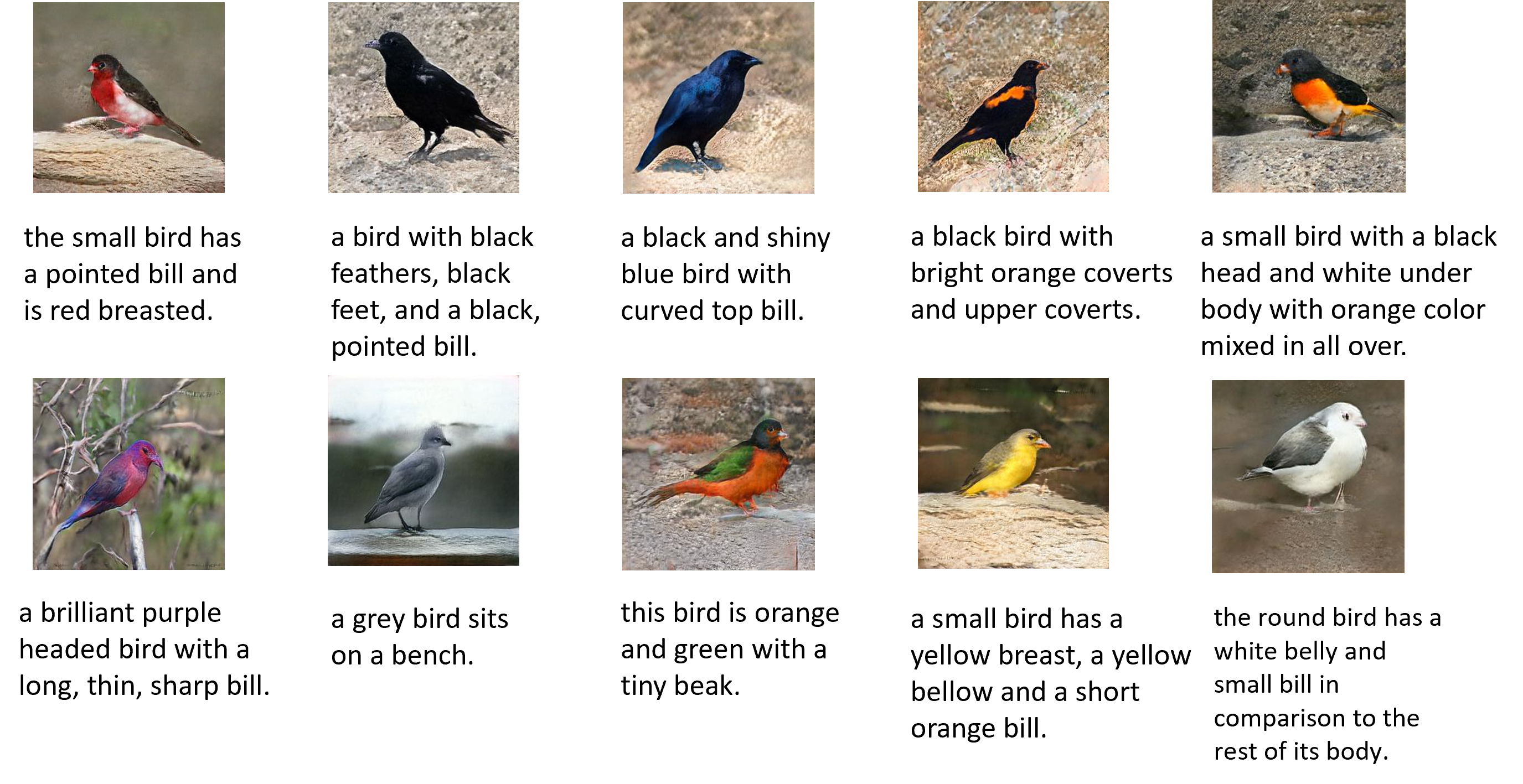}
    \caption{Generating examples on CUB dataset.}
    \label{fig:birds_generated}
\end{figure*}

\begin{figure*}[ht!]
    \centering
    \includegraphics[width=0.7\linewidth]{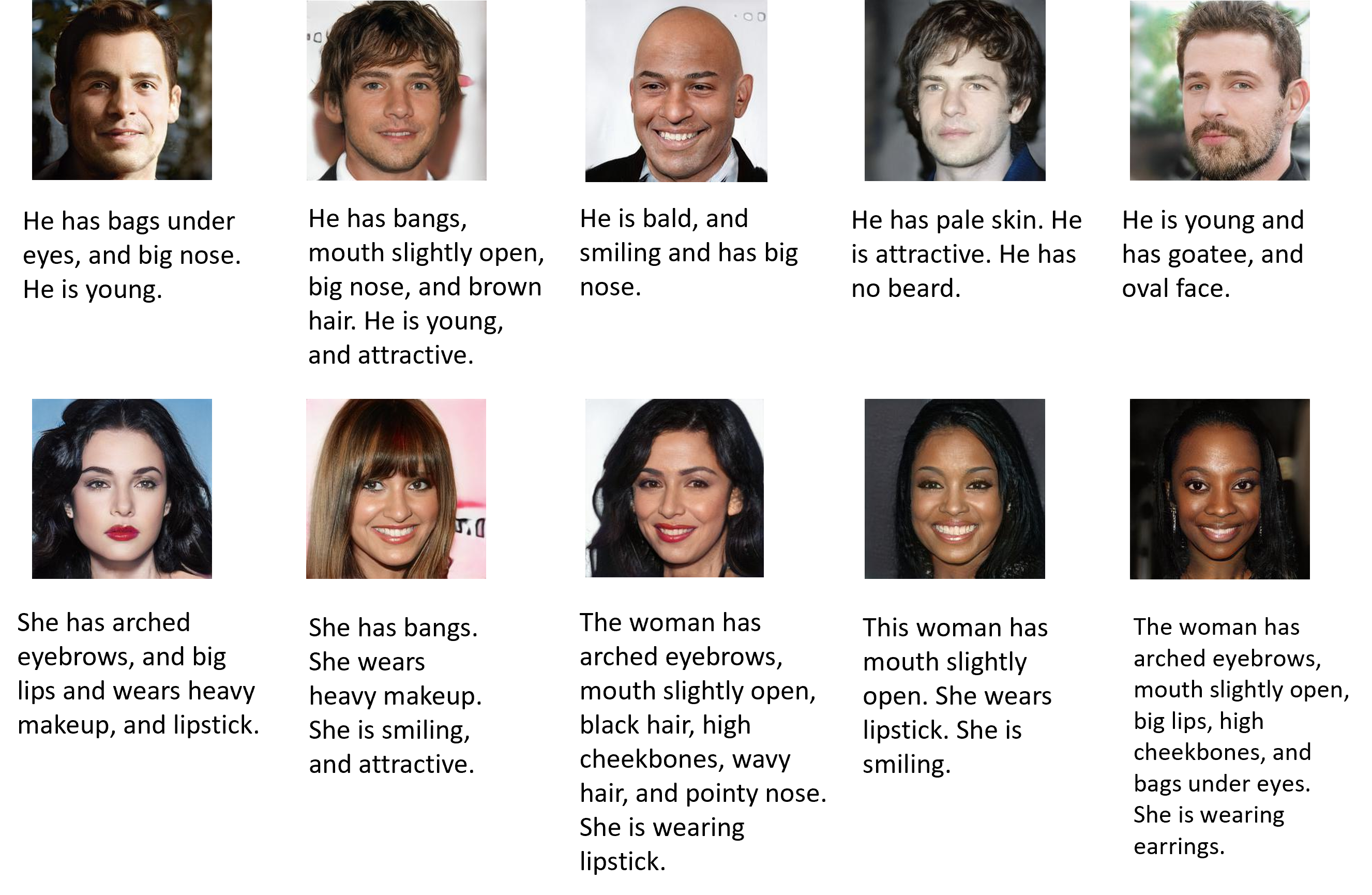}
    \caption{Generating examples on MM CelebA-HQ dataset.}
    \label{fig:celeba_generated}
\end{figure*}

\begin{figure*}[ht!]
    \centering
    \includegraphics[width=0.7\linewidth]{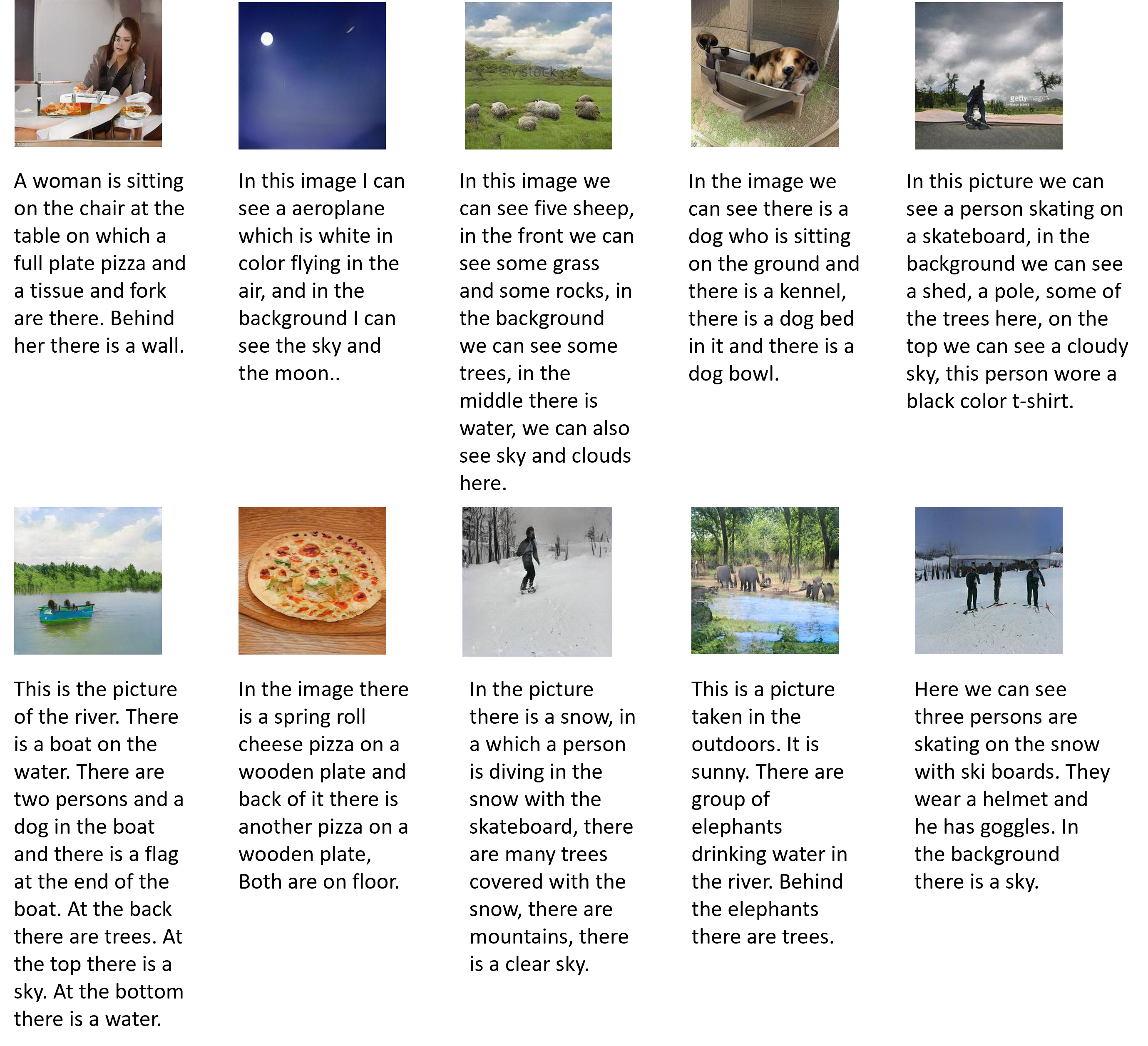}
    \caption{Generating examples on LN-COCO dataset.}
    \label{fig:ln-coco_generated}
\end{figure*}

\begin{figure*}[ht!]
    \centering
    \includegraphics[width=0.6\linewidth]{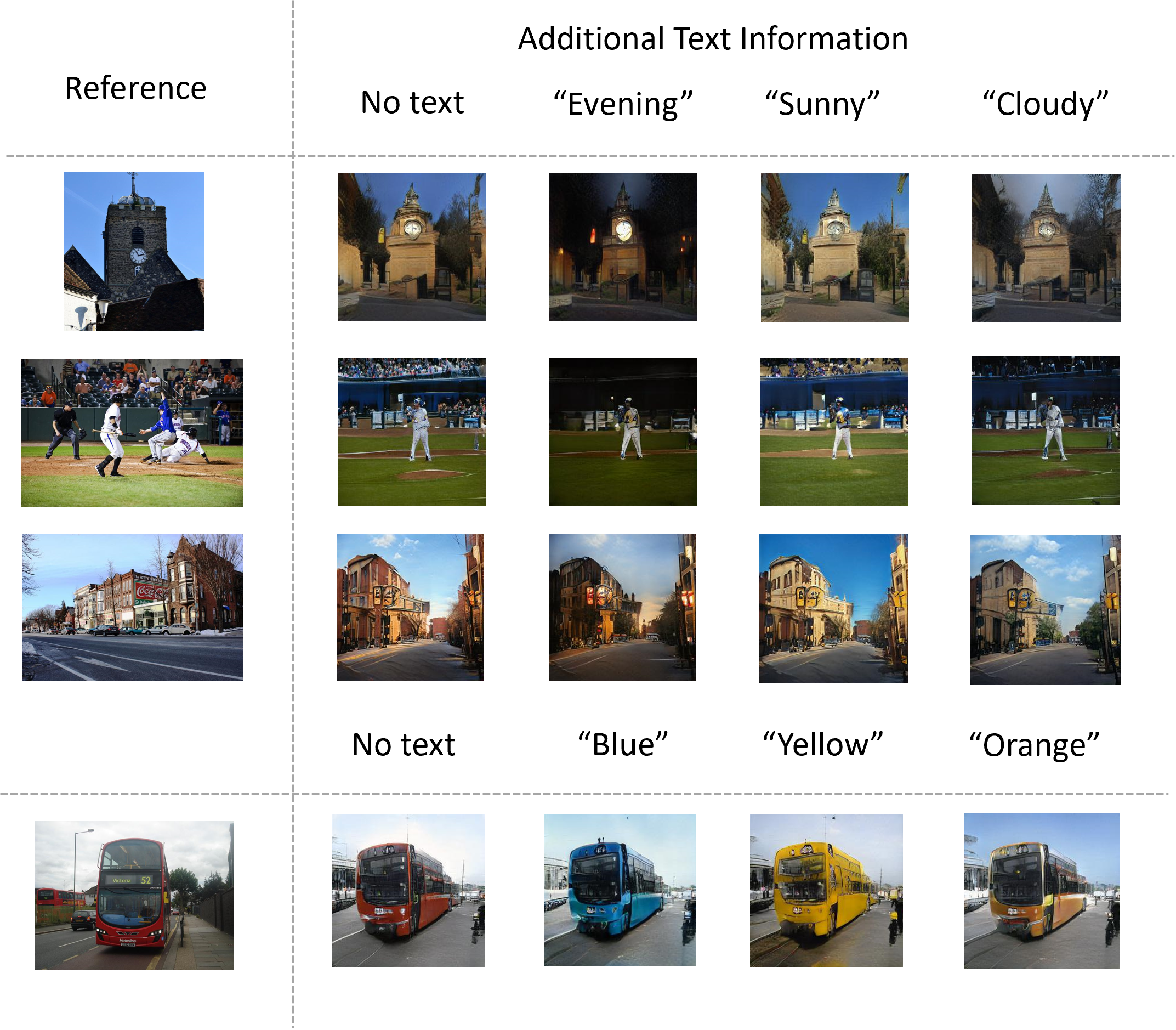}
    \caption{Generating images with multi-modal conditions (conditioned on both image and text) on MS-COCO dataset.}
    \label{fig:ms-coco_mixed}
\end{figure*}

\begin{figure*}[ht!]
    \centering
    \includegraphics[width=0.6\linewidth]{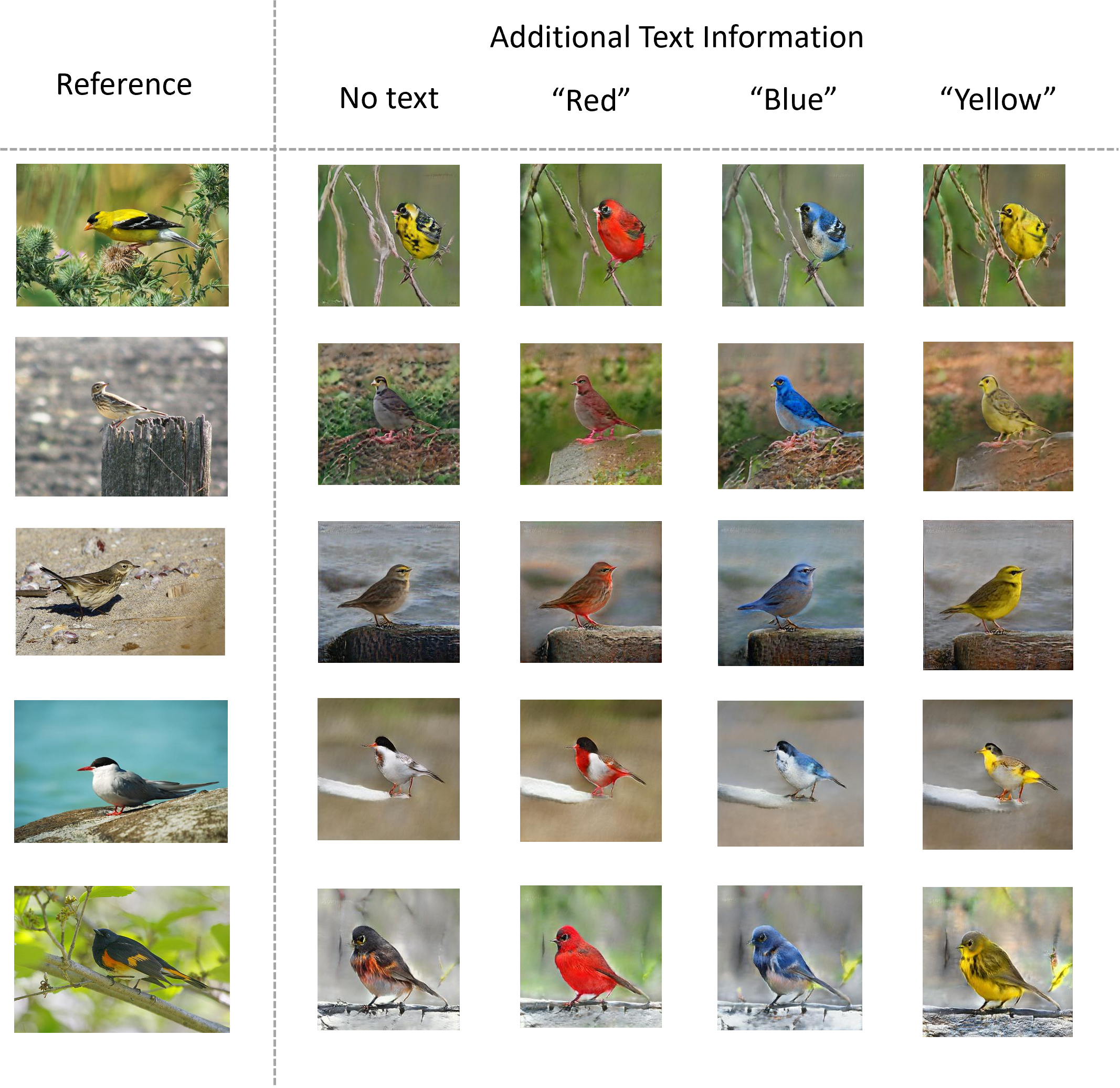}
    \caption{Generating images with multi-modal conditions (conditioned on both image and text) on CUB dataset.}
    \label{fig:birds_mixed}
\end{figure*}

\begin{figure*}[ht!]
    \centering
    \includegraphics[width=0.6\linewidth]{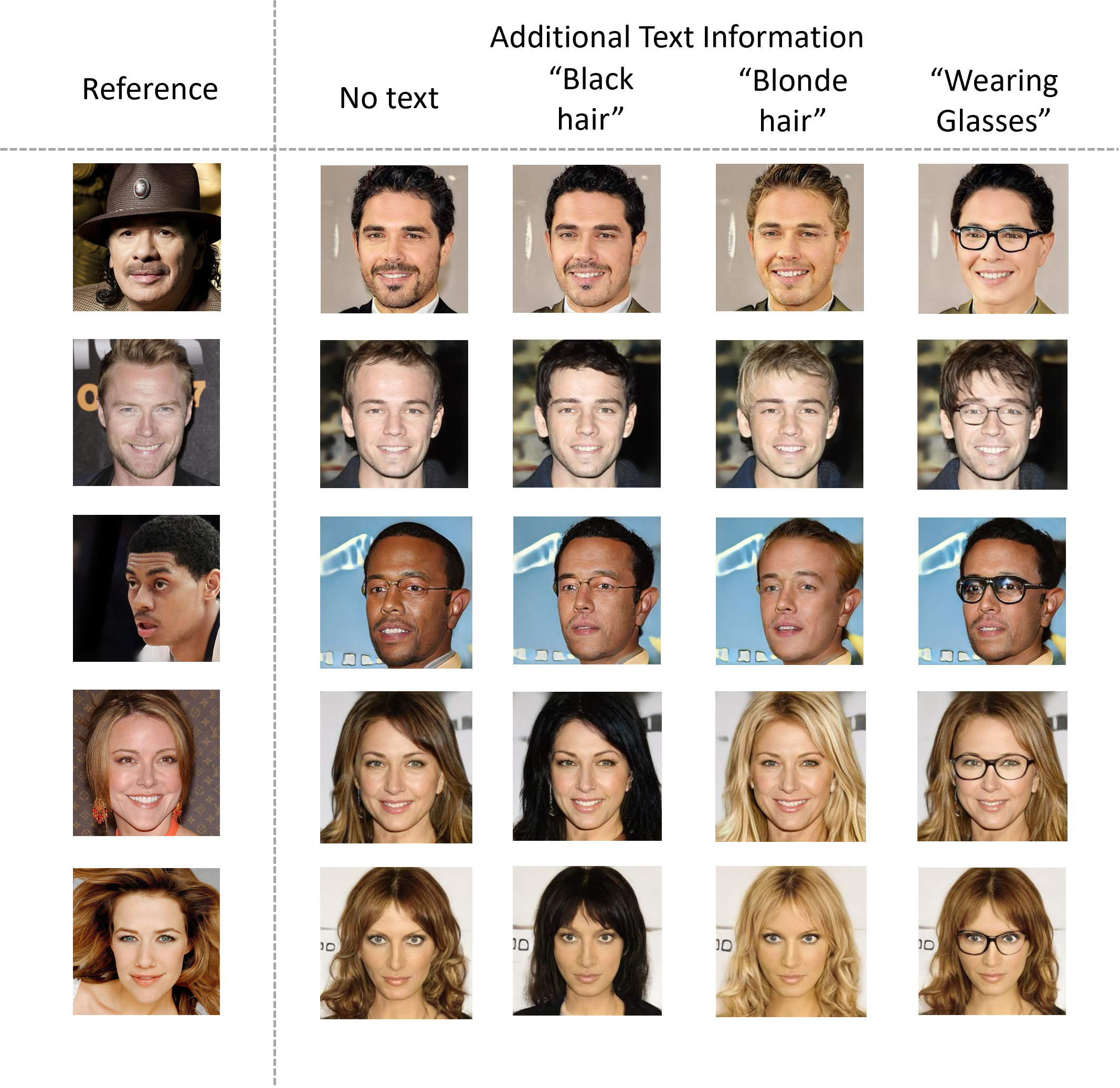}
    \caption{Generating images  with multi-modal conditions (conditioned on both image and text) on MM CelebA-HQ dataset.}
    \label{fig:celeba_mixed}
\end{figure*}

\begin{figure*}[ht!]
    \centering
    \includegraphics[width=0.6\linewidth]{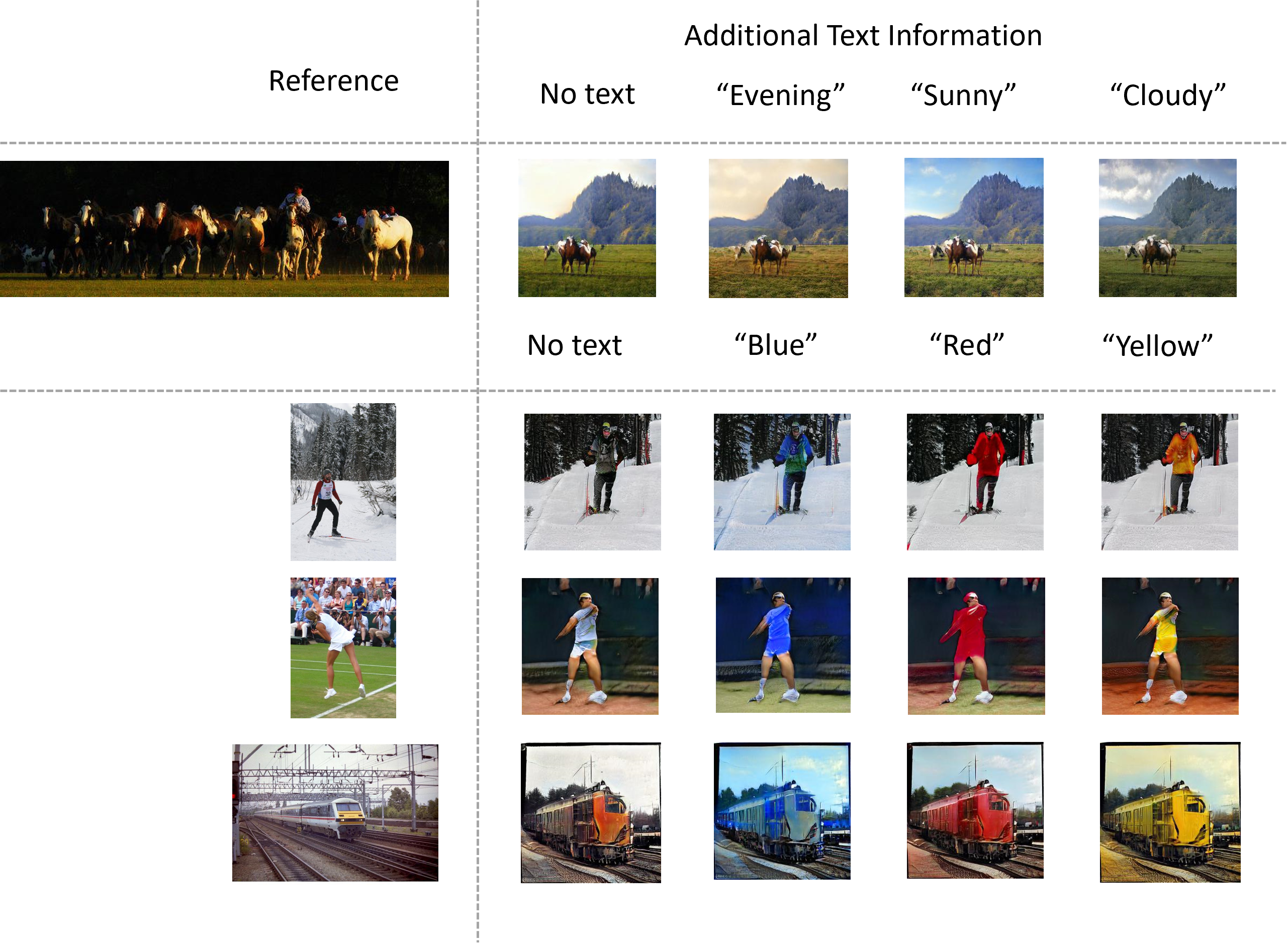}
    \caption{Generating images with multi-modal conditions (conditioned on both image and text) on LN-COCO dataset.}
    \label{fig:ln-coco_mixed}
\end{figure*}

\end{document}